\documentclass[journal]{IEEEtran}

\usepackage[noadjust]{cite}
\usepackage{caption}
\usepackage{subcaption}
\usepackage{graphicx}
\usepackage{color}
\usepackage{placeins}
\usepackage{float}
\usepackage{tabularx,colortbl}
\usepackage[cmex10]{amsmath}
\usepackage{amssymb}
\usepackage{amsbsy}
\usepackage{bm}
\usepackage{amsthm}
\usepackage{dblfloatfix}
\usepackage{bbm}
\usepackage{epstopdf}
\usepackage{breqn}
\usepackage{cite}
\usepackage{framed}
\usepackage{url}
\usepackage{amsfonts}
\usepackage{mathrsfs}
\usepackage{mathtools}

\DeclareGraphicsExtensions{.eps}
\DeclareMathOperator*{\argmax}{arg\,max}
\DeclareMathOperator*{\argmin}{arg\,min}

\DeclareMathOperator*{\hertz}{Hz}

\DeclareMathOperator*{\support}{supp}

\newtheorem{thm}{Theorem}
\newtheorem{lemma}{Lemma}
\newtheorem{corollary}{Corollary}
\newtheorem{prop}{Proposition}
\definecolor{purple1}{RGB}{100, 85, 255}
\definecolor{orange1}{RGB}{233, 146, 0}
\newcommand*{\QEDB}{\hfill\ensuremath{\square}}%

\begin{document}

\title{Robust Estimation of Self-Exciting {Generalized Linear Models} with Application to Neuronal Modeling}

\author{Abbas~Kazemipour,~\IEEEmembership{Student Member,~IEEE,
}~Min~Wu,~\IEEEmembership{Fellow,~IEEE,
}~and~Behtash~Babadi,~\IEEEmembership{Member,~IEEE
}
\thanks{The authors are with the Department of Electrical and Computer Engineering, University of Maryland, College Park, MD 20742 USA (e-mail: kaazemi@umd.edu; minwu@umd.edu; behtash@umd.edu).}
\thanks{This material is based upon work supported in part by the National Science Foundation under Grant No. 1552946.}
\thanks{Corresponding author: B. Babadi (e-mail: behtash@umd.edu).}
\thanks{Copyright (c) 2017 IEEE. Personal use of this material is permitted. However, permission to use this material for any other purposes must be obtained from the IEEE by sending an email to pubs-permissions@ieee.org.}}

\maketitle

\begin{abstract}
We consider the problem of estimating self-exciting generalized linear models from limited binary observations, where the history of the process serves as the covariate. We analyze the performance of two classes of estimators, namely the $\ell_1$-regularized maximum likelihood and greedy estimators, for a canonical self-exciting process and characterize the sampling tradeoffs required for stable recovery in the non-asymptotic regime. Our results extend those of compressed sensing for linear and generalized linear models with i.i.d. covariates to those with highly inter-dependent covariates. We further provide simulation studies as well as application to real spiking data from the mouse's lateral geniculate nucleus and the ferret's retinal ganglion cells which agree with our theoretical predictions. 
\end{abstract}


\IEEEpeerreviewmaketitle

\section{Introduction}

The theory of compressed sensing (CS) has provided a novel framework for measuring and estimating statistical models governed by sparse underlying parameters {\cite{donoho2006compressed, candes2006compressive, candes2006stable, candes2008introduction, needell2009cosamp, bruckstein2009sparse}}. In particular, for linear models with random covariates and sparsity of the parameters, the CS theory provides sharp trade-offs between the number of measurement, sparsity, and estimation accuracy. Typical theoretical guarantees imply that when the number of measurements are roughly proportional to sparsity, then stable recovery of these sparse models is possible.

Beyond those described by linear models, observations from binary phenomena form a large class of data in natural and social sciences. Their ubiquity in disciplines such as neuroscience, physiology, seismology, criminology, and finance has urged researchers to develop formal frameworks to model and analyze these data. In particular, the theory of point processes provides a statistical machinery for modeling and prediction of such phenomena. Traditionally, these models have been employed to predict the likelihood of self-exciting processes such as earthquake occurrences \cite{ogata1988statistical, vere1970stochastic}, but have recently found applications in several other areas. For instance, these models have been used to characterize heart-beat dynamics \cite{barbieri2005point, valenza2013point} and violence among gangs \cite{egesdal}. Self-exciting point process models have also found significant applications in analysis of neuronal data {\cite{brown2001analysis,smith2003estimating, brown2004multiple,  paninski2004maximum, truccolo2005point, paninski2007statistical, pillow2011model}.} 

In particular, point process models provide a principled way to regress binary spiking data with respect to extrinsic stimuli and neural covariates, and thereby forming predictive statistical models for neural spiking activity. Examples include place cells in the hippocampus \cite{brown2001analysis}, spectro-temporally tuned cells in the primary auditory cortex \cite{calabrese2011generalized}, and spontaneous retinal or thalamic neurons spiking under tuned intrinsic frequencies \cite{Borowska,liets2003spontaneous}. Self-exciting point processes have also been utilized in assessing the functional connectivity of neuronal ensembles \cite{kim2011granger,brown_func_conn}. When fitted to neuronal data, these models exhibit three main features: first, the underlying parameters are nearly sparse or compressible  \cite{Brown_pp, brown_func_conn}; second, the covariates are often highly structured and correlated; and third, the input-output relation is highly nonlinear. Therefore, the theoretical guarantees of compressed sensing do not readily translate to prescriptions for point process estimation.

Estimation of these models is typically carried out by Maximum Likelihood (ML) or regularized ML estimation in discrete time, where the process is viewed as a Generalized Linear Model (GLM). In order to adjust the regularization level, empirical methods such as cross-validation are typically employed \cite{brown_func_conn}. In the signal processing and information theory literature, sparse signal recovery under Poisson statistics has been considered in \cite{poi_sp_del} with application to the analysis of ranking data. In \cite{cs_poi}, a similar setting has been studied, with motivation from imaging by photon-counting devices. Finally, in theoretical statistics, high-dimensional $M$-estimators with decomposable regularizers, such as the $\ell_1$-norm, have been studied for GLMs \cite{Negahban}. 

A key underlying assumption in the existing theoretical analysis of estimating GLMs is the independence and identical distribution (i.i.d.) of covariates. This assumption does not hold for self-exciting processes, since the history of the process takes the role of the covariates. Nevertheless, regularized ML estimators show remarkable performance in fitting GLMs to neuronal data with history dependence and highly non-i.i.d. covariates.  In this paper, we close this gap by presenting new results on robust estimation of compressible GLMs, relaxing the common assumptions of i.i.d. covariates and exact sparsity. 

In particular, we will consider a canonical GLM and will analyze two classes of estimators for its underlying parameters: the $\ell_1$-regularized maximum likelihood and greedy estimators. We will present theoretical guarantees that extend those of CS theory and characterize fundamental trade-offs between the number of measurements, model compressibility, and estimation error of GLMs in the non-asymptotic regime. Our results reveal that when the number of measurements scale sub-linearly with the product of the ambient dimension and a generalized measure of sparsity (modulo logarithmic factors), then stable recovery of the underlying models is possible, even though the covariates solely depend on the history of the process. We will further discuss the extensions of these results to more general classes of GLMs. Finally, we will present applications to simulated as well as real data from two classes of neurons exhibiting spontaneous activity, namely the mouse's lateral geniculate nucleus and the ferret's retinal ganglion cells, which agree with our theoretical predictions. Aside from their theoretical significance, our results are particularly important in light of the technological advances in neural prostheses, which require robust neuronal system identification based on compressed data acquisition.

The rest of the paper is organized as follows: In Section \ref{prelim}, we present our notational conventions, preliminaries and problem formulation. In Section \ref{theoretical}, we discuss the estimation procedures and state the main theoretical results. Section \ref{simulations} provides numerical simulations as well as application to real data. In Section \ref{discussions}, we discuss the implications of our results and outline future research directions. Finally, we present the proofs of the main theoretical results and give a brief background on relevant statistical tests in Appendices \ref{appprf}--\ref{appks}.

\section{Preliminaries and Problem Formulation} \label{prelim}
We {first} give a brief introduction to self-exciting GLMs (see \cite{Daley} for a detailed treatment). {We will use the following {notation} throughout the paper. Parameter vectors are {denoted} by bold-face {Greek} letters. For example, $\boldsymbol{\theta}=[\theta_1,\theta_2,\cdots,\theta_p]'$ denotes a $p$-dimensional parameter vector, with $[\cdot]'$ denoting the transpose operator. We also use the notation {$\mathbf{x}_i^j$} to represent the $(j-i+1)$-dimensional vector $[x_i,x_{i+1},\cdots,x_j]'$ for any $i, j \in \mathbb{Z}$ with $i \le j$.} { For a vector $\boldsymbol{\theta}$, we define its decomposition into positive and negative parts given by:
\[
\boldsymbol{\theta} = \boldsymbol{\theta}^+ - \boldsymbol{\theta}^{-},
\]
where $\boldsymbol{\theta}^\pm = \max\{\pm\boldsymbol{\theta}, \mathbf{0}\}$.
It can be shown that
\[
\| \boldsymbol{\theta}^\pm\|_1= \mathbf{1}' \boldsymbol{\theta}^\pm = \frac{\|\boldsymbol{\theta}\|_1 \pm \mathbf{1}'\boldsymbol{\theta}}{2}
\]
are convex in $\boldsymbol{\theta}$.
}

We consider a sequence of observations in the form of binary spike trains obtained by discretizing continuous-time observations (e.g. electrophysiology recordings), using bins of length $\Delta$. We assume that  not more than one event fall into any given bin. In practice, this can always be achieved by choosing $\Delta$ small enough. The binary observation at bin $i$ is denoted by $x_i$. The observation sequence can be modeled as the outcome of conditionally independent Poisson or Bernoulli trials, with a spiking probability given by $\mathbb{P}(x_i = 1) =: \lambda_{i|H_i}$, where $\lambda_{i|H_i}$ is the spiking probability at bin $i$ given the history of the process $H_i$ up to bin $i$.

These models are widely-used in neural data analysis and are motivated by the continuous time point processes with history dependent conditional intensity functions \cite{Daley}. For instance, given the history of a continuous-time point process $H_t$ up to time $t$, a conditional intensity of $\lambda(t | H_t) = \lambda$ corresponds to the {homogeneous Poisson} process.  As another example, a conditional intensity of $\lambda(t | H_t) = \mu + \int_{-\infty}^t \theta(t - \tau) dN(\tau)$ corresponds to a process known as the Hawkes process \cite{Hawkes_orig} with base-line rate $\mu$ and history dependence kernel $\theta(\cdot)$. Under the assumption of the orderliness of a continuous-time point process, a discretized approximation to these processes can be obtained by binning the process by bins of length $\Delta$, and defining the spiking probability by $\lambda_i := \lambda(i \Delta | H_{i\Delta}) \Delta + o(\Delta)$. In this paper, we consider discrete {random processes} characterized by the spiking probability $\lambda_{i|H_i}$, which are either inherently discrete or employed as an approximation to continuous-time point process models.

Throughout the rest of the paper, we drop the dependence of $\lambda_{i|H_i}$ on $H_i$ to simplify notation, denote it by $\lambda_i$ and refer to it {as} spiking probability. Given the sequence of {binary} observed data {$\mathbf{x}_1^n$}, the negative log-likelihood function under the {Bernoulli} statistics can be expressed as:
\begin{equation}
\label{L_def}
\mathfrak{L}(\boldsymbol{\theta}) = -\frac{1}{n} \sum\limits_{i=1}^n \left \{ x_i\log\lambda_{i}+(1-x_i)\log(1-\lambda_{i}) \right \}.
\end{equation}
Another common likelihood model used in the analysis of neuronal spiking data corresponds to Poisson statistics \cite{Brown_pp}, for which the negative log-likelihood takes the following form:
\begin{equation}
\label{poi_ll}
\mathfrak{L}(\boldsymbol{\theta}) :=  -\frac{1}{n} \sum\limits_{i=1}^n \left \{ x_i\log\lambda_{i}-\lambda_{i} \right \}.
\end{equation}
Throughout the paper, we will focus on binary observations governed by Bernoulli statistics, whose negative log-likelihood is given in Eq. (\ref{L_def}). In applications such as electrophysiology in which neuronal spiking activities are recorded at a high sampling rate, the binning size $\Delta$ is very small and the Bernoulli and Poisson statistics coincide.

When the discrete process is viewed as an {approximation} to a continuous-time process, these log-likelihood functions are known as the Jacod log-likelihood approximations \cite{Daley}. We will present our analysis for the negative log-likelihood given by (\ref{L_def}), but our results can be extended to other statistics including (\ref{poi_ll}) (See {the remarks of Section \ref{theoretical}}).

Throughout this paper {${\mathbf{x}_{-p+1}^{n}}$} will be considered as the observed spiking sequence which will be used for estimation purposes. A popular class of models for $\lambda_i$ is given by GLMs. In its general form, a GLM consists of two main components: an observation model and an equation expressing some (possibly nonlinear) function of the observation mean as a \textit{linear} combination of the covariates. In neural systems, the covariates consist of external stimuli as well as the history of the process. Inspired by spontaneous neuronal activity, we consider \textit{fully} self-exciting processes, in which the covariates are only functions of the process history. As for a canonical {GLM} inspired by the Hawkes process, we consider a process for which the spiking probability is a \textit{linear} function of the process history:
\begin{equation}
\label{CIF_models}
\lambda_i := \mu + \boldsymbol{\theta}' {\mathbf{x}_{i-p}^{i-1}},
\end{equation}
where $\mu$ is a positive constant representing the base-line rate, and $\boldsymbol{\theta}=[\theta_1,\theta_2,\cdots,\theta_p]'$ is a parameter vector denoting the history dependence of the process. {We further assume that the process is non-degenerate, i.e., it will not terminate in an infinite sequence of zeros.} We refer to this {GLM, viewed as a random process,} as the \textit{canonical self-exciting process}. Other popular models in the computational neuroscience literature include the log-link model where $\lambda_i = \exp ( \mu + \boldsymbol{\theta}'{\mathbf{x}_{i-p}^{i-1}} )$ and the logistic-link model where $\lambda_i = \frac{\exp(\mu + \boldsymbol{\theta}'{\mathbf{x}_{i-p}^{i-1}})}{1+\exp (\mu + \boldsymbol{\theta}'{\mathbf{x}_{i-p}^{i-1}})}$. The parameter vector $\boldsymbol{\theta}$ can be thought of as the binary equivalent of autoregressive (AR) parameters in linear AR models.

When fitted to neuronal spiking data, the parameter vector $\boldsymbol{\theta}$ exhibits a degree of sparsity \cite{Brown_pp, brown_func_conn}, that is, only certain lags in the history have a significant contribution in determining the statistics of the process. These lags can be thought of as the preferred or intrinsic {delays} in the spontaneous response of a neuron. {To be more precise, for a sparsity level $s < p$, we denote by $\boldsymbol{\theta}_s$ the best $s$-term approximation to $\boldsymbol{\theta}$. We also define
\begin{equation}
\sigma_s(\boldsymbol{\theta}) := \|\boldsymbol{\theta}-\boldsymbol{\theta}_s\|_1,
\end{equation}  
which is a scalar function of $\boldsymbol{\theta}$ and $s$, and} captures the compressibility of the parameter vector $\boldsymbol{\theta}$ in the $\ell_1$ sense. For a fixed $\xi \in (0,1)$, we say that $\boldsymbol{\theta}$ is \emph{$(s,\xi)$-compressible} if $\sigma_s(\boldsymbol{\theta}) = \mathcal{O}(s^{1-\frac{1}{\xi}})$ \cite{needell2009cosamp}. Note that when $\xi = 0$, the parameter vector $\boldsymbol{\theta}$ is exactly $s$-sparse.

Finally, in this paper, we are concerned with the compressed sensing regime where $n \ll p$, i.e., the observed data has a much smaller length than the ambient dimension of the parameter vector. The main estimation problem of this paper is the following: \emph{given observations ${\mathbf{x}_{-p+1}^n}$ from the \textcolor{black}{canonical self-exciting process}, the goal is to estimate the unknown baseline rate $\mu$ and the $p$-dimensional $(s,\xi)$-compressible history dependence parameter vector $\boldsymbol{\theta}$ in a stable fashion (where the estimation error is controlled) when $n \ll p$.}

\section{Theoretical Results}\label{theoretical}
In this section, we consider two estimators for $\boldsymbol{\theta}$, namely, the $\ell_1$-regularized ML estimator and a greedy estimator, and present the main theoretical results of this paper on the estimation error of these estimators. Note that when $\mu$ is not known, the following results can be applied to the augmented parameter vector $[\mu,\boldsymbol{\theta}']'$. We analyze the case of known $\mu$ for simplicity of presentation.

\subsection{$\ell_1$-Regularized ML Estimation}

The natural estimator for the parameter vector is the ML estimator, which is widely used in neuroscience \cite{Brown_pp}, which by virtue of (\ref{L_def}) is given by:
\begin{equation}
\label{ML_est_pp_L}
\widehat{\boldsymbol{\theta}}_{{\sf ML}}=\argmin\limits_{\boldsymbol{\theta}\in \boldsymbol{\Theta} } \mathfrak{L}(\boldsymbol{\theta}),
\end{equation}
where $\boldsymbol{\Theta}$ is the relaxed closed convex feasible region for which $0 \le \lambda_i \le 1$ given by the conditions:
\vspace{-1mm}
\begin{align}\label{eq:star}
\begin{tabular}{l} ${0 < \pi_{\min} \le \mu -\mathbf{1}' \boldsymbol{\theta}^-}$,\\
 ${\mu +\mathbf{1}' \boldsymbol{\theta}^+ \leq \pi_\max < 1/2}$,
 \end{tabular} {\tag{$\star$}}
\end{align}
for some constants $\pi_{\min}$ and $\pi_{\max}$. {This first inequality incurs minimal loss of generality, as $\pi_{\min}$ can be chosen to be arbitrarily small. The restriction of $\pi_{\max} < 1/2$ ensures that the process is fast mixing and has mainly been adopted for technical convenience. This assumption incurs some loss of generality, as it excludes processes for which the maximum spiking probability exceeds $1/2$. However, due to the low spiking probability of typical neuronal activity, this loss is tolerable for the applications of interest in this paper (see Section \ref{simulations})}.

In the regime of interest when $n \ll p$, the ML estimator is ill-posed and is typically regularized with a smooth norm. In order to capture the compressibility of the parameters, we consider  the $\ell_1$-regularized ML estimator:
\begin{equation}
\label{sp_est_pp_L}
\widehat{\boldsymbol{\theta}}_{{\sf sp}}:=\argmin\limits_{\boldsymbol{\theta}\in \boldsymbol{\Theta}} \quad \mathfrak{L}(\boldsymbol{\theta})+ \gamma_n\|\boldsymbol{\theta}\|_1.
\end{equation}
where $\gamma_n > 0$ is a regularization parameter. It is easy to verify that {the objective function and constraints in Eq.} (\ref{sp_est_pp_L}) are convex in $\boldsymbol{\theta}$ and hence $\widehat{\boldsymbol{\theta}}_{\sf sp}$ can be obtained using standard numerical solvers. Note that the solution to (\ref{sp_est_pp_L}) might not be unique. However, we will provide error bounds that hold for all possible solutions of (\ref{sp_est_pp_L}), with high probability.

It is known that ML estimates are asymptotically unbiased under mild conditions, and with $p$ fixed, the solution converges to the true parameter vector as $n \rightarrow \infty$. However, it is not clear how fast the convergence rate is for finite $n$ or when $p$ is not fixed and is allowed to scale with $n$. This makes the analysis of ML estimators, and in general regularized M-estimators, very challenging \cite{Negahban}. Nevertheless, such an analysis has significant practical implications, as it will reveal sufficient conditions on $n$ with respect to $p$ as well as a criterion to choose $\gamma_n$, which result in a stable estimation of $\boldsymbol{\theta}$. Finally, note that we are fixing the ambient dimension $p$ throughout the analysis. In practice, the history dependence is typically negligible beyond a certain lag and hence for a large enough $p$, {GLMs} fit the data very well.

\subsection{Greedy Estimation}
Although there exist fast solvers to convex problems of the type given by Eq. (\ref{sp_est_pp_L}), these algorithms are polynomial time in $n$ and $p$, and may not scale well with high-dimensional data.  This motivates us to consider greedy solutions for the estimation of $\boldsymbol{\theta}$. In particular, we will consider a generalization of the Orthogonal Matching Pursuit (OMP) \cite{zhang_omp,OMP} for general convex cost functions. A flowchart of this algorithm is given in Table \ref{tab:1}, which we denote by the Point Process Orthogonal Matching Pursuit (POMP) algorithm. At each iteration, the component in which the objective function has the largest deviation is chosen and added to the current support. The algorithm proceeds for a total of $s^{\star}$ steps, resulting in an estimate with $s^\star$ components.

The main idea behind the generalized OMP is in the greedy selection stage, where the absolute value of the gradient of the cost function at the current solution is considered as the selection metric. Consider an estimate $\widehat{\boldsymbol{\theta}}^{(k-1)}$ at the $(k-1)$-st stage of the generalized OMP for a quadratic cost function of the form $\| \mathbf{b} - \mathbf{A} \boldsymbol{\theta}\|_2^2$, with $\mathbf{b}$ and $\mathbf{A}$ denoting the observation vector and covariates matrix, respectively. Then, the gradient takes the form $\mathbf{A}' (\mathbf{b} - \mathbf{A} \widehat{\boldsymbol{\theta}}^{(k-1)})$ which is exactly the correlation vector between the residual error and the columns of $\mathbf{A}$ as in the original OMP algorithm. 

\begin{table}
\centering
\framebox{$\begin{array}{l}
\text{Input: } \mathfrak{L}(\boldsymbol{\theta}) , s^\star\\
\text{Output: } \widehat{\boldsymbol{\theta}}_{\sf POMP}^{(s^\star)}\\
\text{Initialization:}\Big\{\begin{array}{l}
\text{Start with the index set } S^{(0)}=\emptyset\\
\text{and the initial estimate }\widehat{\boldsymbol{\theta}}^{(0)}_{{\sf POMP}} = 0
\end{array}\\
\textbf{for } k=1,2,\cdots,s^\star\\
\text{  }\begin{array}{l}
j = \argmax \limits_i \left| \left( \nabla \mathfrak{L} \; \left(\widehat{\boldsymbol{\theta}}_{{\sf POMP}}^{(k-1)}\right) \right)_i\right|\\
S^{(k)}=S^{(k-1)}\cup \{j\}\\
\widehat{\boldsymbol{\theta}}_{{\sf POMP}}^{(k)} = \argmin \limits_{\support ({\boldsymbol{\theta}}) \subset S^{(k)}} \mathfrak{L}(\boldsymbol{\theta})
\end{array}\\
\textbf{end }\\
\end{array}$
}
\caption{\small{Point Process Orthogonal Matching Pursuit (POMP)}}
\label{tab:1}
\end{table}

\subsection{Theoretical Guarantees}

Recall that the parameter vector $\boldsymbol{\theta} \in \mathbb{R}^p$ is assumed to be $(s,\xi)$-compressible, so that $\sigma_s(\boldsymbol{\theta}) = \|\boldsymbol{\theta}-\boldsymbol{\theta}_S\|_1 = \mathcal{O} (s^{1-\frac{1}{\xi}})$, and the observed data are given by the vector ${\mathbf{x}_{-p+1}^n} \in \{0,1\}^{n+p-1}$, all in the regime of $s, n \ll p$. In the remainder of this paper, we assume that $\boldsymbol{\theta} \in \boldsymbol{\Theta}$. The main theoretical result regarding the performance of the $\ell_1$-regularized ML estimator is given by the following theorem:
\begin{thm}
\label{negahban_lambda}
If $\sigma_s(\boldsymbol{\theta}) = \mathcal{O}(\sqrt{s})$, there exist constants $d_1,d_2,d_3$ and $d_4$ such that for $n>d_1s^{2/3}p^{2/3} \log p$ and a choice of $\gamma_n=d_2 \sqrt{\frac{\log p}{n}}$, any solution $\widehat{\boldsymbol{\theta}}_{{\sf sp}}$ to (\ref{sp_est_pp_L}) satisfies the bound
\begin{equation}
\label{thm2}
\left \|\widehat{\boldsymbol{\theta}}_{{\sf sp}}-\boldsymbol{\theta}\right \|_2 \leq d_3 \sqrt{\frac{s \log p}{n}}+ \sqrt{d_3\sigma_s(\boldsymbol{\theta})}\sqrt[4]{\frac{\log p}{n}},
\end{equation}
with probability greater than $1-\mathcal{O}\left(\frac{1}{n^{d_4}}\right)$.
\end{thm}

Similarly, the following theorem characterizes the performance bounds for the POMP estimate:

\begin{thm}
\label{thm_OMP}
If $\boldsymbol{\theta}$ is $(s,\xi)$-compressible for some $\xi < 1/2$, there exist constants $d_1',d_2',d_3'$ and $d_4'$ such that for $n>d_1's^{2/3} p^{2/3} \left(\log s \right)^{2/3} \log p$, the POMP estimate satisfies the bound
\begin{equation}
\label{thm2_OMP}
\left \|\widehat{\boldsymbol{\theta}}_{{\sf POMP}}-\boldsymbol{\theta}\right\|_2 \leq d_2' \sqrt{\frac{s \log s \log p}{n}} + d_3' \frac{\log s}{s^{{\frac{1}{\xi}-2}}}
\end{equation}
after $s^\star =\mathcal{O}(s \log s)$ iterations with probability greater than $1-\mathcal{O}\left(\frac{1}{n^{d'_4}}\right)$.
\end{thm}

{Full proofs of Theorems \ref{negahban_lambda} and \ref{thm_OMP} are given in Appendix \ref{appprf}.}

\noindent  \textit{\textbf{Remarks.}} An immediate comparison of the sufficient condition $n = \mathcal{O}(s^{2/3} p^{2/3} \log p)$ of Theorem \ref{negahban_lambda} with those of \cite{Negahban} for GLM models with i.i.d. covariates given by $n = \mathcal{O}(s \log p)$ reveals that a loss of order $\mathcal{O}(p^{2/3} s^{-1/3})$ is incurred due to the inter-dependence of the covariates. However, the sample space of $n$ i.i.d. covariates is $np$-dimensional, whereas in our problem the sample space is only $(n+p)$-dimensional. Hence, the aforementioned loss can be viewed as the price of self-averaging of the process accounting for the low-dimensional nature of the covariate sample space. \textcolor{black}{To the best of our knowledge, the dominant loss of $\mathcal{O}(p^{2/3})$ in both theorems does not seem to be significantly improvable, as self-exciting processes are known to converge quite slowly to their ergodic state \cite{reynaud2007some}. {On a related note, the analysis of the sampling requirements of linear AR models reveals a loss of {$\mathcal{O}(p^{1/2})$} in the number of measurements \cite{arpaper}.}}

The sufficient condition of Theorem \ref{thm_OMP} given by $n = \mathcal{O}(s^{2/3} p^{2/3} \left(\log s \right)^{2/3} \log p)$ implies an extra loss of $\left(\log s \right)^{2/3}$ due to the greedy nature of the solution. Theorem \ref{thm_OMP} also requires a high compressibility level of the parameter vector $\boldsymbol{\theta}$  ($\xi < 1/2$), whereas Theorem \ref{negahban_lambda} does not impose any extra restrictions on $\xi \in (0,1)$. Intuitively speaking, this comparison reveals the trade-off between computational complexity and compressibility requirements for convex optimization vs. greedy techniques, which is well-known for linear models \cite{bruckstein2009sparse}.

The constants $d_i, d'_i$, $i=1,\cdots,4$, $\alpha$ and $\beta$ are explicitly given in the proof of the theorems in Appendix \ref{appprf}. As for a typical numerical example, for $\pi_\min = 0.01$ and $\pi_\max=0.49$, the constants of Theorem \ref{negahban_lambda} can be chosen as $d_1 \approx 10^3, d_2 = 50, d_3 \approx 10^4$ and $d_4 = 4$. We will next give a sketch of the proof of these theorems.

\noindent \textit{\textbf{Proof Sketches of Theorems 1 and 2.}} The main ingredient in the proofs of Theorems \ref{negahban_lambda} and \ref{thm_OMP} is inspired by the beautiful treatment of Negahban et al. in \cite{Negahban} in establishing the notion of Restricted Strong Convexity (RSC). By the convexity of the negative Jacod log-likelihood given by Eq. (\ref{L_def}), it is clear that a small change in $\boldsymbol{\theta}$ results in a small change in the negative Jacod log-likelihood. However, the converse is not necessarily true. Intuitively speaking, the RSC condition guarantees that the converse holds: a small change in the log-likelihood implies a small change in the parameter vector, i.e., the log-likelihood is not too \textit{flat} around the true parameter vector. A depiction of the RSC condition for $p=2$, adopted from \cite{Negahban}, is given in Figure \ref{rscfig}. In Figure \ref{rscfig}(a), the RSC does not hold since a change along $\theta_2$ does not change the log-likelihood, whereas the log-likelihood in Figure \ref{rscfig}(b) satisfies the RSC.

\begin{figure}[h]
\centering
\includegraphics[width=.95\columnwidth]{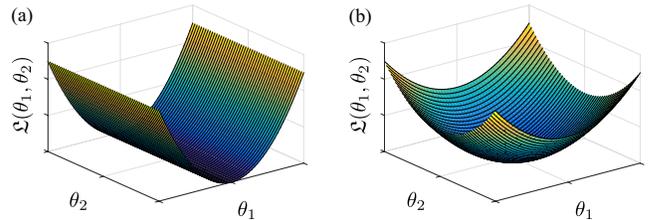}
\caption{\small{Illustration of RSC (a) RSC does not hold (b) RSC does hold.}}\label{rscfig}
\vspace{-2mm}
\end{figure}
More formally, if the log-likelihood is twice differentiable at $\boldsymbol{\theta}$, the RSC is equivalent to existence of a lower quadratic bound on the negative log-likelihood:
\begin{equation}
\label{RSC}
\mathfrak{D_L}({\boldsymbol{\psi}},\boldsymbol{\theta}):= \mathfrak{L}(\boldsymbol{\theta}+{\boldsymbol{\psi}})-  \mathfrak{L}(\boldsymbol{\theta})-{\boldsymbol{\psi}}'\nabla\mathfrak{L}(\boldsymbol{\theta})\geq \kappa\|{\boldsymbol{\psi}}\|_2^2,
\end{equation}

for a positive constant $\kappa > 0$ and all { ${\boldsymbol{\psi}}\in\mathbb{R}^p$} in a carefully-chosen neighborhood of $\boldsymbol{\theta}$ depending on $s$ and $\xi$. Based on the results of \cite{Negahban} and \cite{zhang_omp}, when the RSC is satisfied, sufficient conditions akin to those in Theorems \ref{negahban_lambda} and \ref{thm_OMP} can be obtained by estimating the {Euclidean} extent of the solution set around the true parameter vector (see Propositions \ref{negahban_thm} and \ref{prop_omp} in Appendix \ref{appprf}).

The major technical challenge for the \textcolor{black}{canonical self-exciting process}, as opposed to the GLM models with i.i.d. covariates in \cite{Negahban}, lies in the fact that the covariates are highly inter-dependent as they are formed by the history of the process. Hence, it is not straightforward to establish RSC with high probability, as the large deviation techniques used for i.i.d. random vectors {do} not hold. We establish the RSC for the \textcolor{black}{canonical self-exciting process} in two steps (see Lemma \ref{lemma1} in Appendix \ref{appprf}). First, we show that RSC holds for the expected value of the negative log-likelihood $\mathbb{E} [ \mathfrak{L}(\boldsymbol{\theta})]$, and then by invoking results on concentration of dependent random variables show that the negative log-likelihood $\mathfrak{L}(\boldsymbol{\theta})$ resides in a sufficiently small neighborhood of $\mathbb{E} [ \mathfrak{L}(\boldsymbol{\theta})]$ with high probability, and hence satisfies the RSC.

The remainder of the proof of Theorem \ref{negahban_lambda} establishes that upon satisfying the RSC, the estimation error can be suitably bounded (Proposition \ref{negahban_thm}, Appendix \ref{appprf}). Similarly, Theorem 2 is proven using the RSC together with the results adopted from \cite{zhang_omp} on the performance of OMP for convex cost functions (Proposition \ref{prop_omp}, Appendix \ref{appprf}).

\noindent \textit{\textbf{Extensions.}} For simplicity and clarity of presentation, we have opted to present the proofs for the case of known $\mu$ and for the {canonical self-exciting process}. The following corollary extends our results to the case of unknown $\mu$:

\begin{corollary}\label{cor:1} The claims of Theorems \ref{negahban_lambda} and \ref{thm_OMP} hold when $\mu$ is not known, except for possibly slightly different constants.
\end{corollary}
\begin{proof}
The proof is given in Appendix \ref{appprf}.
\end{proof}

The {canonical self-exciting process} can be generalized to a larger class of {GLMs} by generalization of its spiking probability function. In a more general form we can consider a spiking probability function given by
\begin{equation*}
\label{nhp}
\lambda_i = \phi\left(\mu+\boldsymbol{\theta}'{\mathbf{x}_{i-p}^{i-1}}\right),
\end{equation*}
where $\phi(\cdot)$ is a possibly nonlinear function for which $0 < \lambda_i < 1$. In their continuous form, such processes are referred to as the \textit{nonlinear Hawkes process} \cite{nlhp}. Two of the commonly-used models in neural data analysis are the log-link and logistic-link models. Our prior numerical studies in \cite{kazemipour} revealed a similar performance improvement of  the $\ell_1$-regularized ML and the greedy solution over the ML estimate for the log-link model. Stationarity of these discrete processes can be proved similar to the canonical self-exciting process (see Appendix \ref{app:hawkes_psd}). The latter fact is key to extending our proofs to other models and is summarized by the following corollary:

\begin{corollary}\label{cor:2} Theorems \ref{negahban_lambda} and \ref{thm_OMP} hold when the spiking probability is given by $\lambda_i = \phi\left(\mu+\boldsymbol{\theta}'{\mathbf{x}_{i-p}^{i-1}}\right)$ for some continuous, bounded, convex and twice-differentiable function $\phi(\cdot)$ (e.g., $\phi(x) = \exp(x)$ or $\phi(x) = {\sf logit}^{-1}(x)$)  for which $0 < \lambda_i < 1/2$, except for  different constants.
\end{corollary}
\begin{proof}
The proof is given in Appendix \ref{appprf}.
\end{proof}

\section{Application to Simulated and Real Data}\label{simulations}

In this section, we study the performance of the conventional ML estimator, the $\ell_1$-regularized ML estimator, and the POMP estimator on simulated data as well as real spiking data recorded from the mouse's lateral geniculate nucleus (LGN) neurons and the ferret's retinal ganglion cells (RGC). {We have archived a MATLAB implementation of the estimators used in this paper using the CVX package \cite{cvx} on the open source repository GitHub and made it publicly available \cite{hawkes_code}.}

\vspace{-3mm}
\subsection{Simulation Studies}

In order to simulate spiking data governed by the canonical self-exciting process, we sequentially generate spikes using (\ref{CIF_models}).   We have used $\mu = 0.1$, $\pi_\min = 0.01$, $\pi_\max = 0.49$, $p=1000$, $s=3$ and $n = 950$ for simulation purposes. Figure \ref{hawkes_fig} shows $500$ samples of the {canonical self-exciting process} generated using a history dependence parameter vector shown in Figure \ref{estimate_plot_synthetic}(a). The parameter vector $\boldsymbol{\theta}$ is compressible with a sparsity level of $s = 3$ and $\sigma_3(\boldsymbol{\theta}) = 0.05$. A value of $\gamma_n = 0.1$ is used to obtain the $\ell_1$-regularized ML estimate, which is slightly tuned around the theoretical estimate given by Theorem \ref{negahban_lambda}. Figures \ref{estimate_plot_synthetic}(b), \ref{estimate_plot_synthetic}(c), and \ref{estimate_plot_synthetic}(d) show the estimated history dependence parameter vectors using ML, $\ell_1$-regularized ML, and POMP, respectively. It can be readily visually observed that regularized ML and POMP significantly outperform the ML estimate in finding the correct values of $\boldsymbol{\theta}$. More specifically, the components at lags $405$ and $800$ (indicated by the gray arrows) are underestimated by the ML estimator, and their contribution is distributed among several falsely identified smaller lag components.

\begin{figure}[h]
\vspace{-2mm}
\centering
\includegraphics[width=.95\columnwidth]{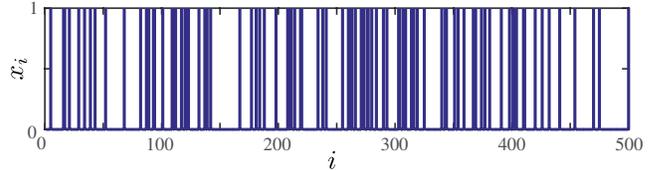}
\caption{{\small{A sample of the simulated {canonical self-exciting process}.}}}\label{hawkes_fig}
\vspace{-3mm}
\end{figure}

\begin{figure}[h]
\begin{center}
\includegraphics[width=.9\columnwidth]{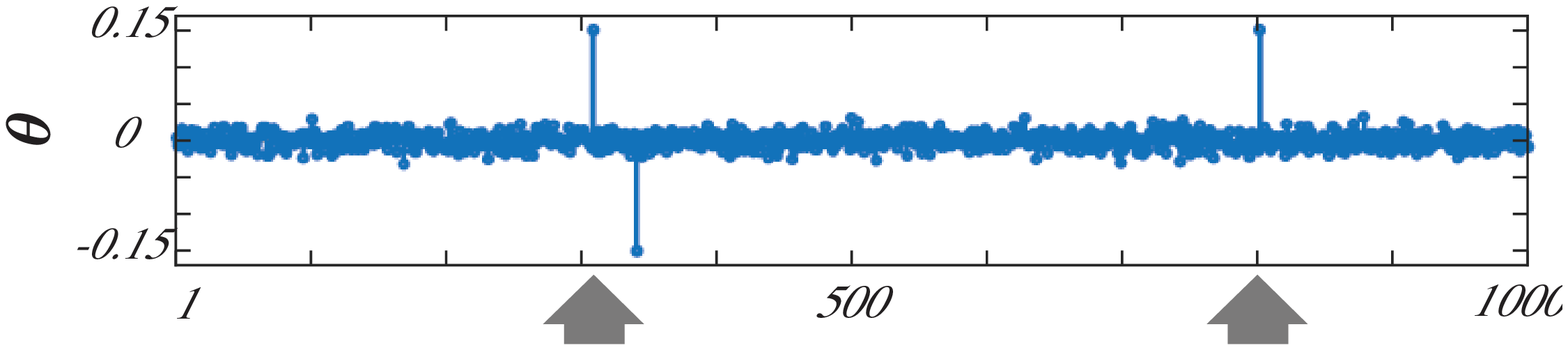}
\subcaption*{(a) True}
\vspace{3mm}
\includegraphics[width=.9\columnwidth]{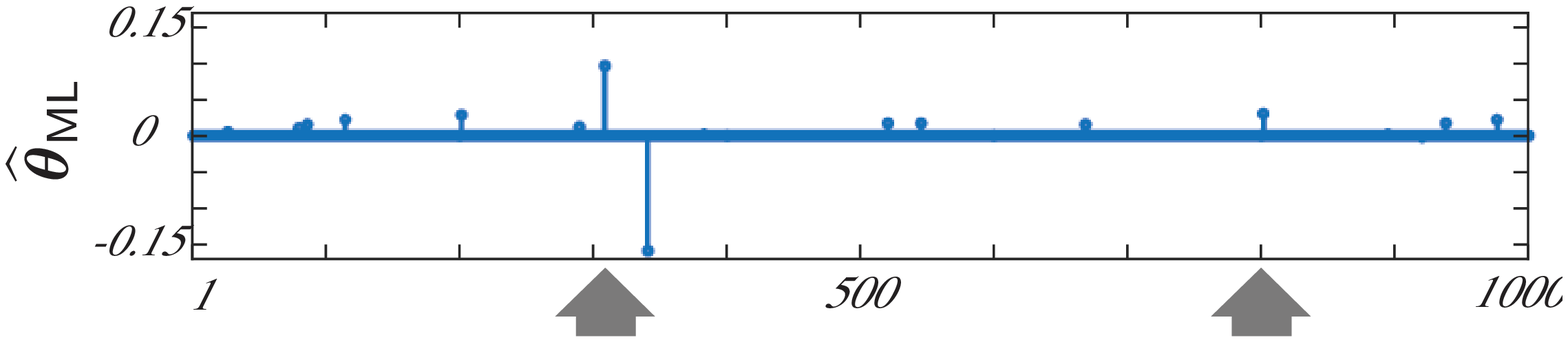}
\subcaption*{(b) ML}
\vspace{3mm}
\includegraphics[width=.9\columnwidth]{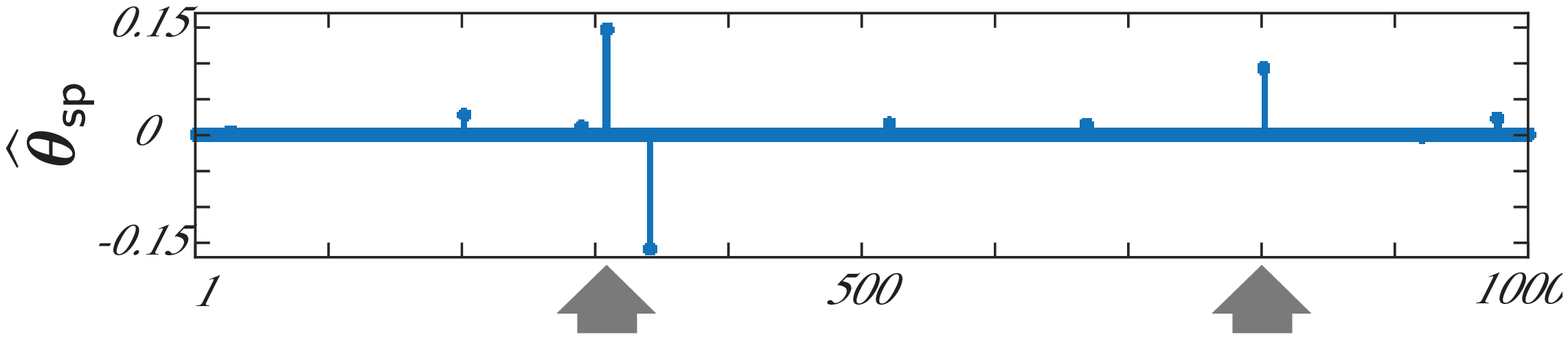}
\subcaption*{(c) $\ell_1$-regularized ML}
\vspace{3mm}
\includegraphics[width=.9\columnwidth]{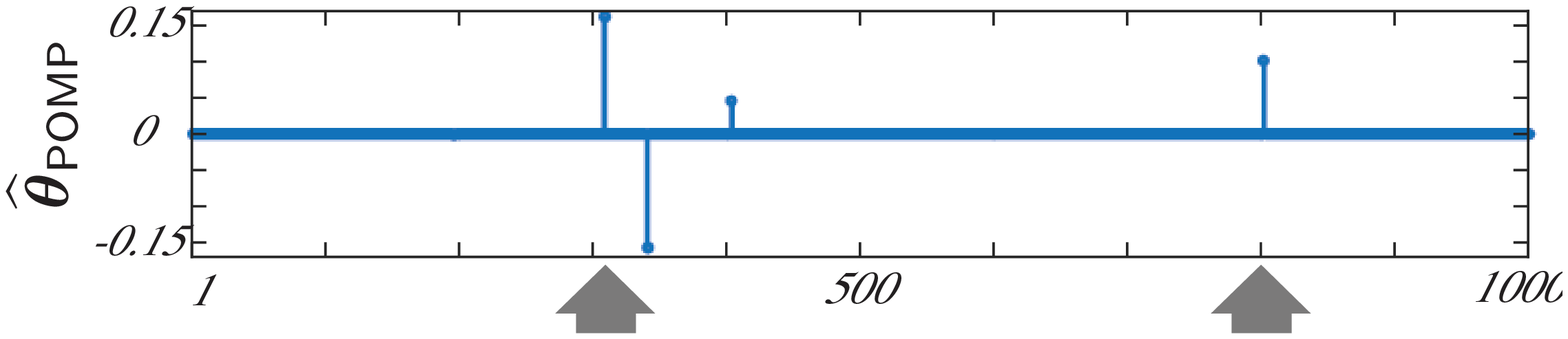}
\subcaption*{(d) POMP}
\end{center}
\vspace{-3mm}
\caption{{\small{(a) True parameters vs. (b) ML, (c) $\ell_1$-regularized ML, and (d) POMP estimates.}}}\label{estimate_plot_synthetic}
\vspace{-2mm}
\end{figure}

\begin{figure}[h]
\centering
\includegraphics[width=.82\columnwidth]{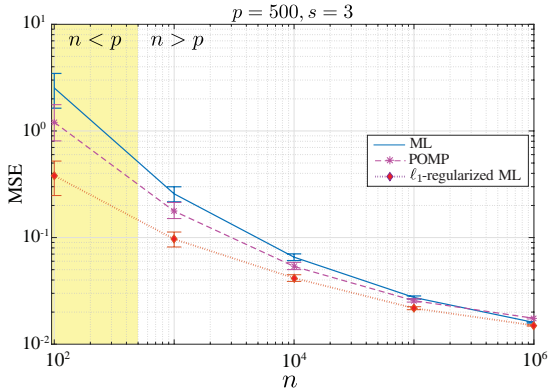}
\caption{{\small{MSE performance of the ML, $\ell_1$-regularized ML and POMP estimators.}}}\label{MSE_high_dim}
\vspace{-5mm}
\end{figure}

In order to quantify this performance gain, we repeated the same experiment {by generating  realizations corresponding to randomly chosen supports of size $s=3$ for $\boldsymbol{\theta}$ and spike trains of length $10^2 \le n \le 10^6$. In each case, the magnitudes of the components of $\boldsymbol{\theta}$ were chosen to satisfy the assumptions (\ref{eq:star}). For a given $\boldsymbol{\theta}$, the mean-square-error (MSE) of the estimate $\widehat{\boldsymbol{\theta}}$ is defined as $\widehat{\mathbb{E}}\{ \| \widehat{\boldsymbol{\theta}} - \boldsymbol{\theta} \|_2^2 \}$, where $\widehat{\mathbb{E}} \{ \cdot \}$ is the sample average over the realizations of the process.} Figure \ref{MSE_high_dim} shows the results of this simulation, where a similar systematic performance gain is observed. The left segment of the plot (shaded in yellow) and the right segment correspond to the compressive ($n < p$) and denoising ($n > p$) regimes, respectively. Error bars on the plot indicate 90\% quantiles of the MSE for this simulation obtained by multiple realizations. As it can be inferred from Figure \ref{MSE_high_dim}, the $\ell_1$-regularized ML and POMP have a systematic performance gain over the ML estimate {in the compressive regime, where $n \ll p$}, with the former outperforming the rest. In the denoising regime, the performance of the $\ell_1$-regularized and ML become closer, while the POMP saturates to a higher MSE floor. The latter observation can be explained by the fact that the POMP can only estimate $s^{\star}$ components (including those of $\boldsymbol{\theta}_s$), and fails to capture the $(p - s^{\star})$ compressible components. This results in an MSE floor above that obtained by ML, for large values of $n$.

\begin{figure}[h]
\begin{center}
\includegraphics[width=.9\columnwidth]{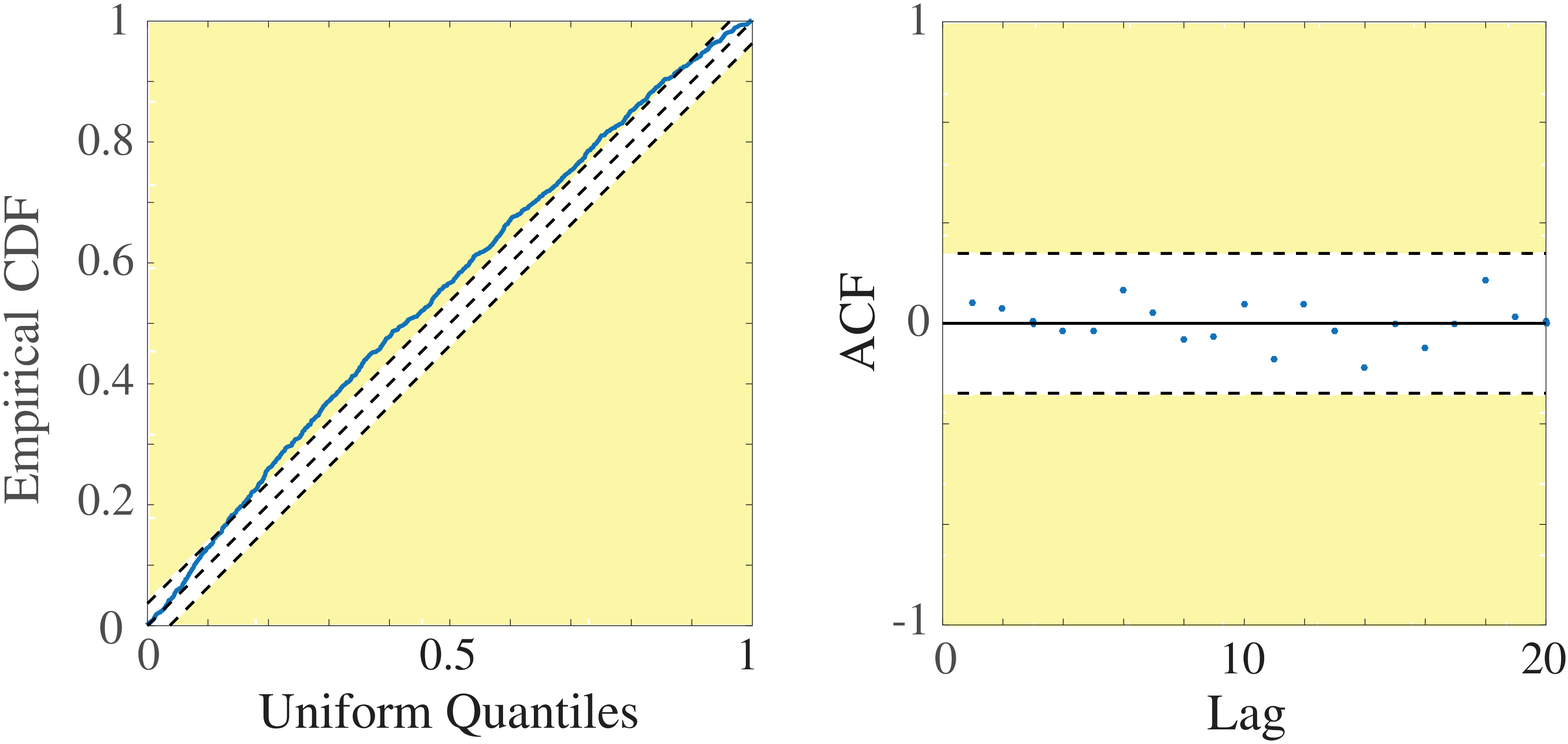}
\subcaption*{(a) ML}
\vspace{3mm}
\includegraphics[width=.9\columnwidth]{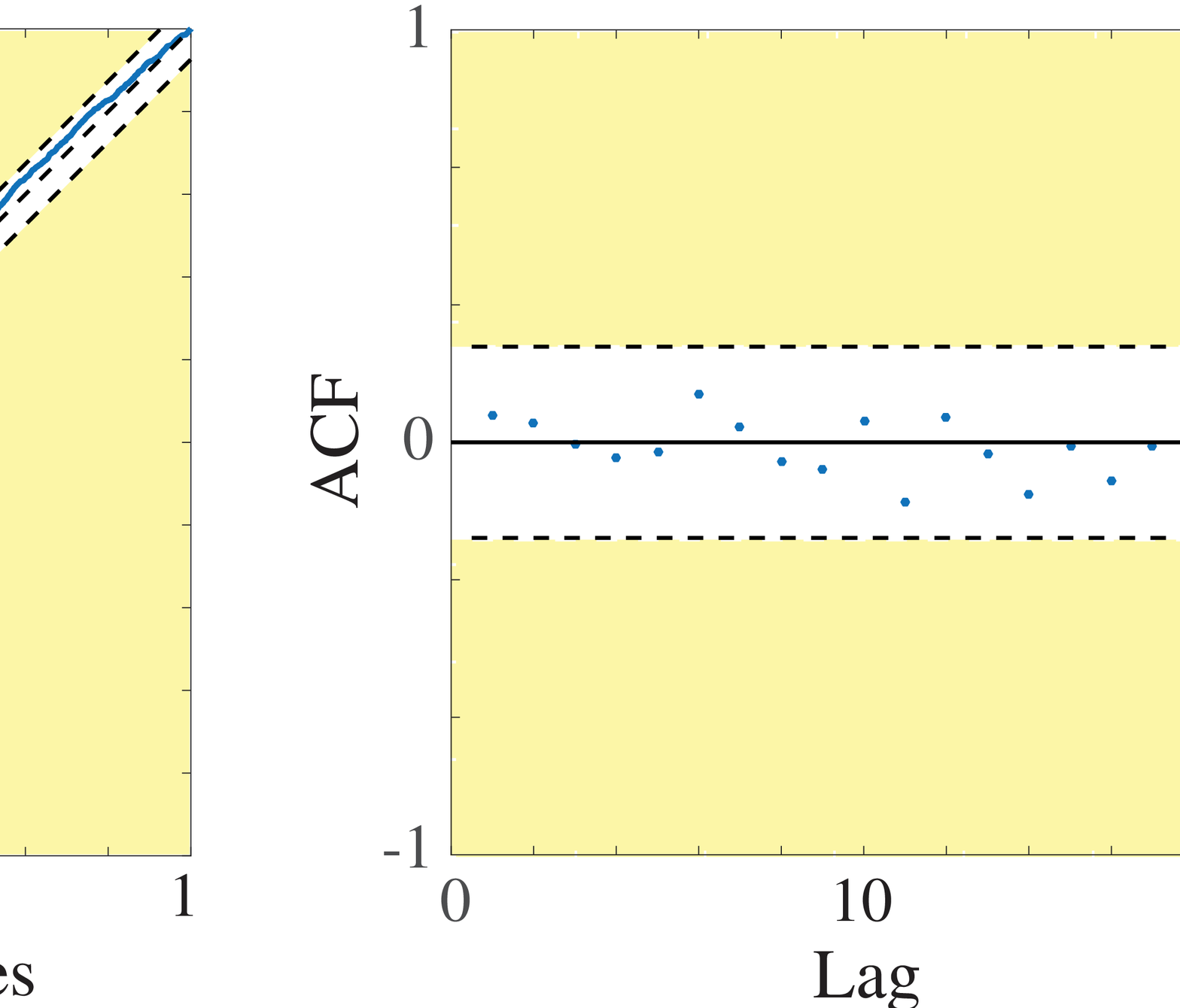}
\subcaption*{(b) $\ell_1$-regularized ML}
\vspace{3mm}
\includegraphics[width=.9\columnwidth]{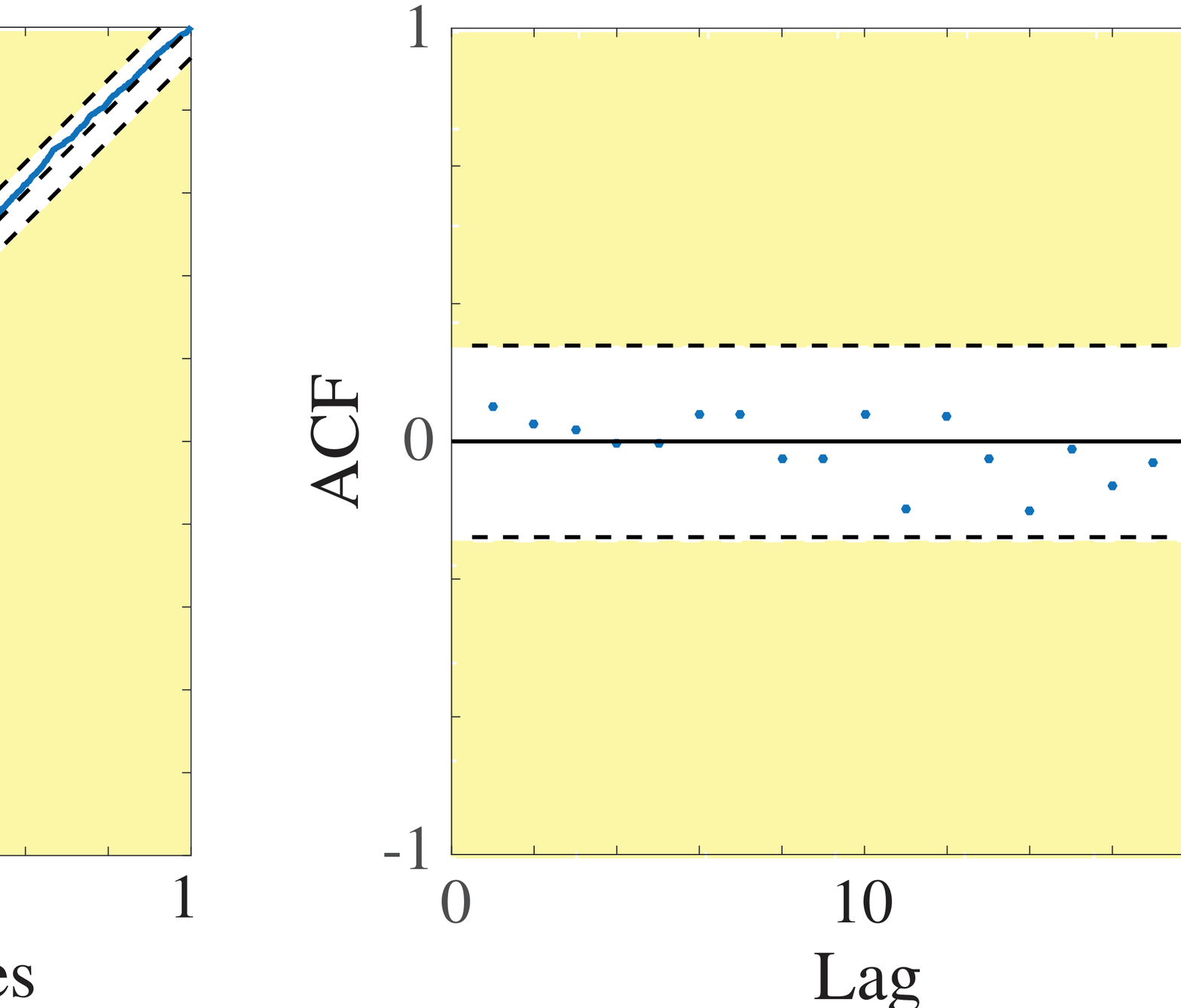}
\subcaption*{(c) POMP}
\end{center}
\vspace{-3mm}
\caption{{\small{KS and ACF tests at {$95\%$ confidence level}, for the ML, $\ell_1$-regularized ML and POMP estimates.}}}
\label{hawkes_ks_acf_synthetic}
\vspace{-5mm}
\end{figure}

{The MSE comparison in Figure \ref{MSE_high_dim}} requires one to know the true parameters. In practice, the true parameters are unknown, and statistical tests are typically used to assess the goodness-of-fit of the estimates to the observed data. We use the Kolmogorov-Smirnov (KS) test and the autocorrelation function (ACF) test to assess the goodness-of-fit. These tests are based on the time-rescaling theorem for point processes \cite{time_rescaling}, which states that if the time axis is rescaled using the {estimated} conditional intensity function of {the} inhomogeneous Poisson process, the resulting point process is a homogeneous Poisson process with unit rate. Thereby, one can test for the validity of the time-rescaling theorem via two statistical tests: the KS test reveals how close the empirical quantiles of the time-rescaled point process to the true quantiles of a unit rate Poisson process, and the ACF test reveals how close the ISI {distribution} of the time-rescaled process is to the true ISI distribution of a unit rate Poisson process. Details of these tests are given in Appendix \ref{appks}. Figure \ref{hawkes_ks_acf_synthetic} shows the KS and ACF tests {at a $95\%$ confidence level} for the ML $\ell_1$-regularized ML, and the POMP estimates from Figure \ref{estimate_plot_synthetic}. The yellow shades mark the regions below the specified confidence levels. The ML estimate {fails to pass the KS test, while the regularized and POMP estimates pass both tests.}

\subsection{Application to Spontaneous Neuronal Spiking Activity}

\subsubsection{{\textbf{Background and motivation}}}
Early studies of spontaneous neuronal activity from the cat's cochlear nucleus \cite{gerstein1960approach} marked a significant breakthrough in computational neuroscience by going beyond the so-called Poisson hypothesis, by which single neurons were assumed to be firing according to homogeneous Poisson statistics. The diversity of the ISIs deduced from the spontaneous activity of the cochlear neurons led to the development of more sophisticated statistical models based on renewal process theory, resulting in the Gamma and inverse Gaussian ISI descriptions of spontaneous neuronal activity \cite{gerstein1964random,tuckwell2005introduction}. Due to the analytical difficulties involved in working with these models, their generalization to a broader range of spiking statistics is not straightforward. 

In light of the more recent discoveries on the role of spontaneous neuronal activity in brain development \cite{xu2011instructive,blankenship2010mechanisms}, its relation to functional architecture \cite{tsodyks1999linking}, and its functional significance in a variety of modalities including retinal \cite{xu2011instructive}, visual \cite{luck1997neural}, auditory \cite{tritsch2007origin}, hippocampal \cite{sombati1995recurrent}, cerebellar \cite{aizenman1999regulation}, and thalamic \cite{pinault1998intracellular} function, the modeling and analysis of this phenomenon has sparked a renewed interest among researchers in recent years. In particular, models based on GLMs have shown to overcome the analytical difficulties of the abovementioned models based on renewal theory, and have been successfuly used in relating the spontaneous neuronal activity to instrinsic and extrinsic neural covariates \cite{paninski2004maximum,paninski2007statistical,time_rescaling,barbieri2001construction} as well as inferring the functional connectivity of neuronal ensembles \cite{kim2011granger,brown_func_conn}. The above-mentioned results rely on the accuracy of the ML estimation of these models. In addition, the estimated parameters are typically sparse. Therefore, the $\ell_1$-regularized ML and POMP estimators are expected to offer a more robust alternative than the ML, especially under the limited observation setting.

{In order to evaluate the performance of these estimators on real data, in the remainder of this section we will compare the performance of the ML, $\ell_1$-regularized ML, and POMP estimators in modeling the spontaneous spiking activity recorded from two different types of neurons, namely the mouse's lateral geniculate nucleus and the ferret's retinal ganglion cells.}

In the following analysis, the regularization parameter $\gamma_n$ was chosen using a two-fold cross-validation refinement around the value obtained from our theoretical results. {The length of the history components $p$ was chosen by first selecting a large enough $p$ as an upper bound for the expected correlation length of neuronal spontaneous activity (estimated as $\sim 1.5~\text{s}$), followed by reducing $p$ to the point where an increase in the history length  does not result in significantly detected history components.}

\subsubsection{{\textbf{Application to LGN spiking activity}}}
\label{nsp}
 We first compare the performance of the estimators on the LGN neurons.  The LGN is part of the thalamus in the brain, which acts as a relay from the retina to the primary visual cortex \cite{thalamus}. The data {were} recorded at $1ms$ resolution from the mouse LGN neurons using single-unit recording \cite{scholl2013emergence}. We used about $5$ seconds of data from one neuron for the analysis. In order to capture the history dependence governing the spontaneous spiking activity of the LGN neuron, we model the spiking probability using the \textcolor{black}{canonical self-exciting process} model with $p=100$ ($\Delta = 1ms$). {Figure \ref{real_fig} shows the spiking data used in the analysis.}

\begin{figure}[h]
\centering
\includegraphics[width=.9\columnwidth]{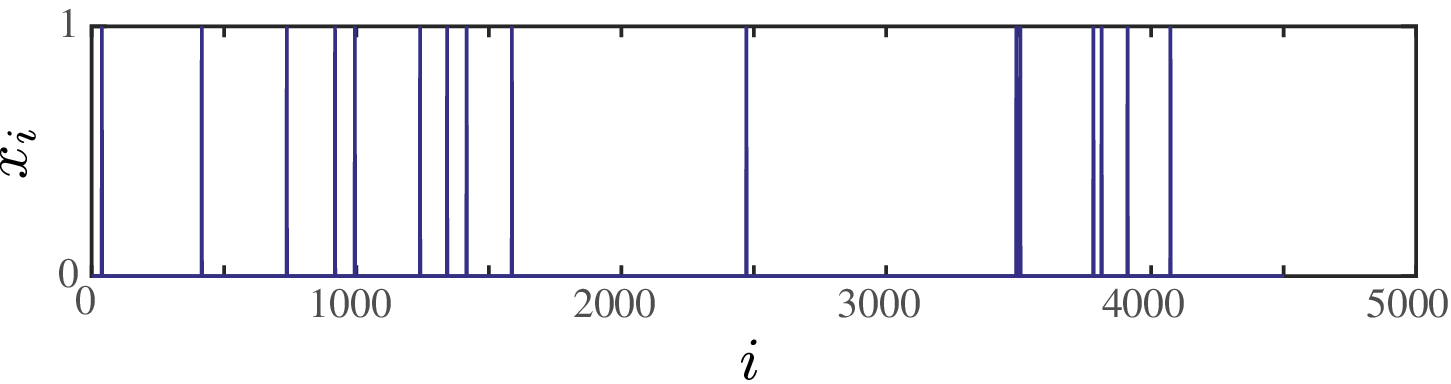}
\caption{\small{The LGN spiking data used in the analysis.}}\label{real_fig}
\vspace{-1mm}
\end{figure} 

\begin{figure}[h]
\begin{center}
\includegraphics[width=.9\columnwidth]{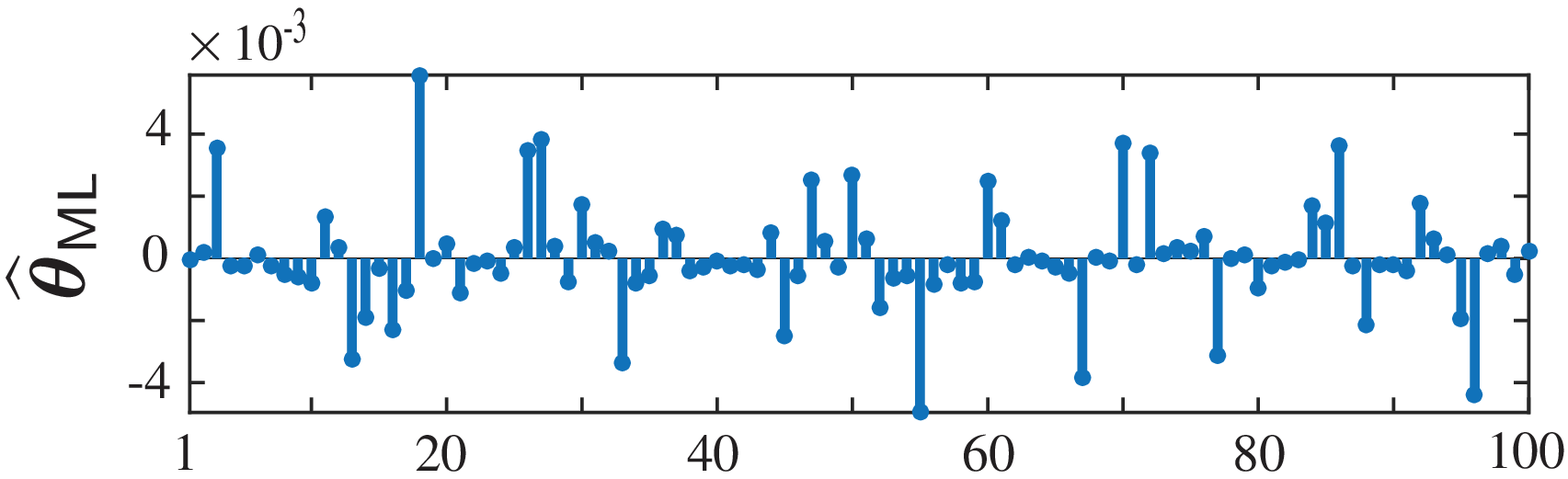}
\subcaption*{(a) ML}
\vspace{-1mm}
\includegraphics[width=.9\columnwidth]{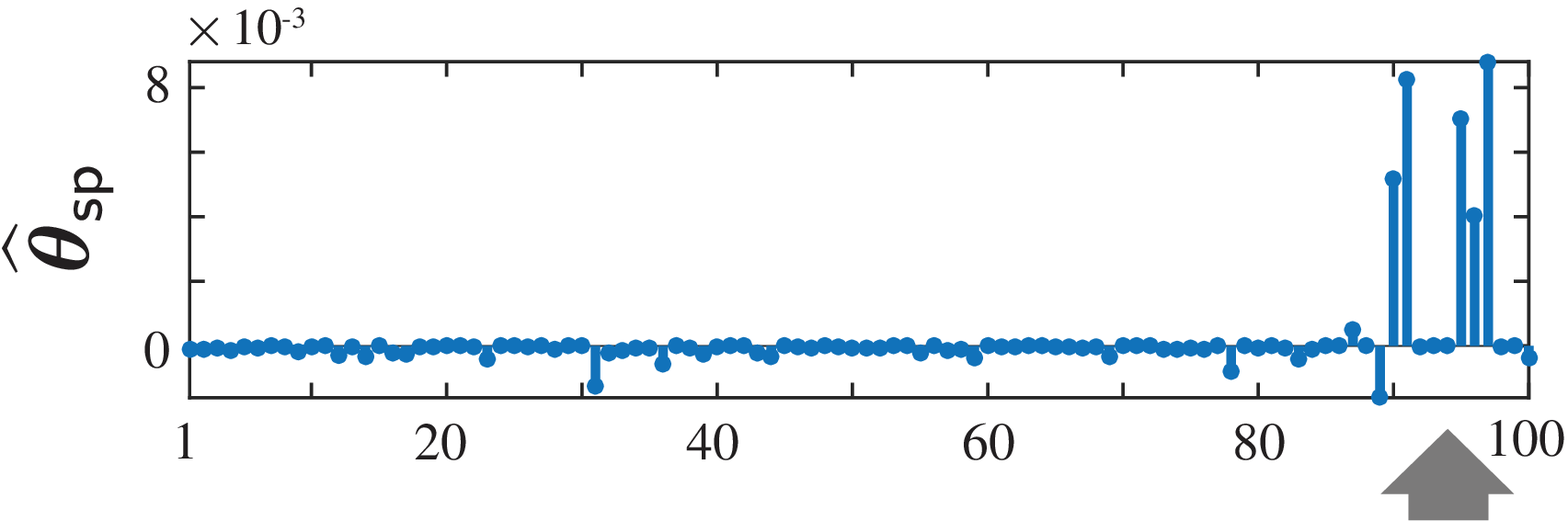}
\subcaption*{(b) $\ell_1$-regularized ML}
\vspace{3mm}
\includegraphics[width=.9\columnwidth]{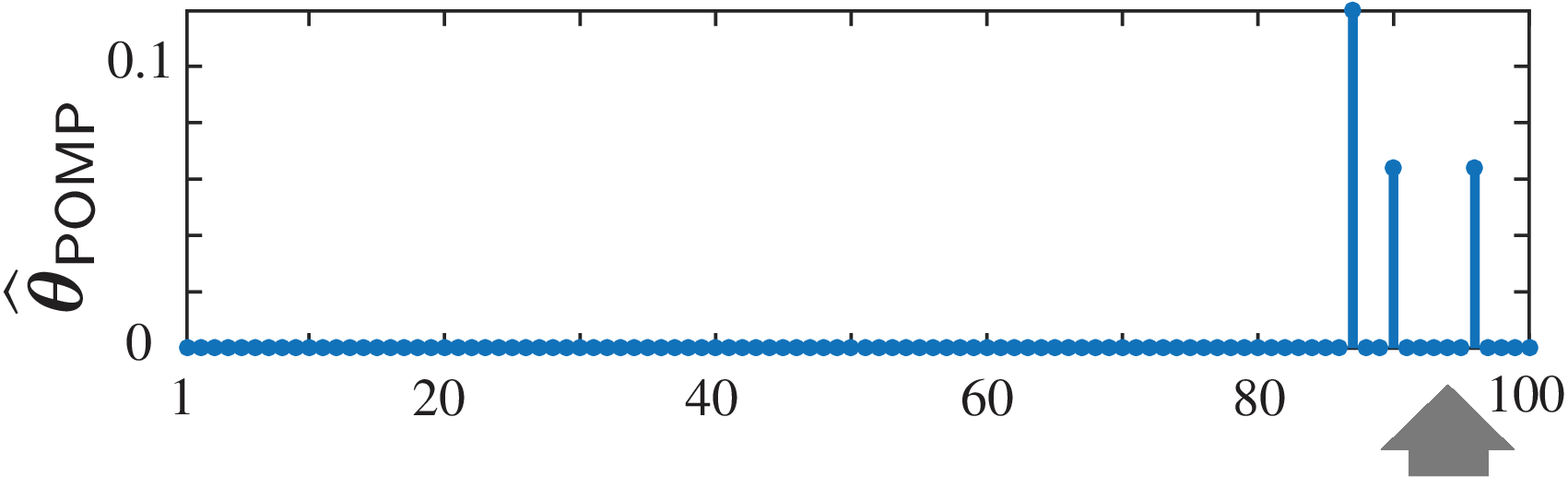}
\subcaption*{(c) POMP}
\end{center}
\vspace{-3mm}
\caption{\small{(a) ML, (b) $\ell_1$-regularized ML, and (c) POMP estimates of the LGN spiking parameters.}}
\label{lgn_mlvssp}
\vspace{-5mm}
\end{figure}

\begin{figure}[h]
\begin{center}
\includegraphics[width=.9\columnwidth]{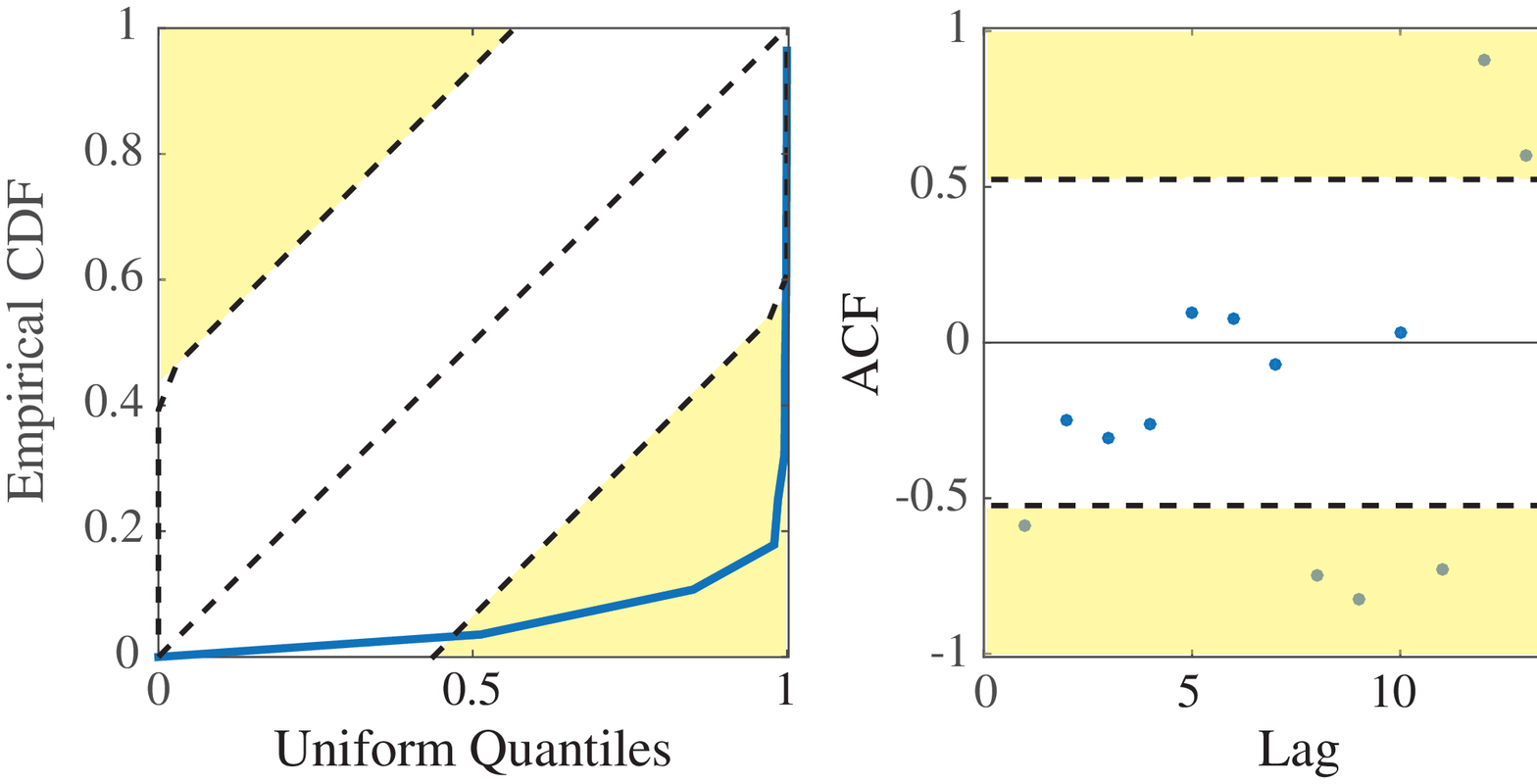}
\subcaption*{(a) ML}
\vspace{3mm}
\includegraphics[width=.9\columnwidth]{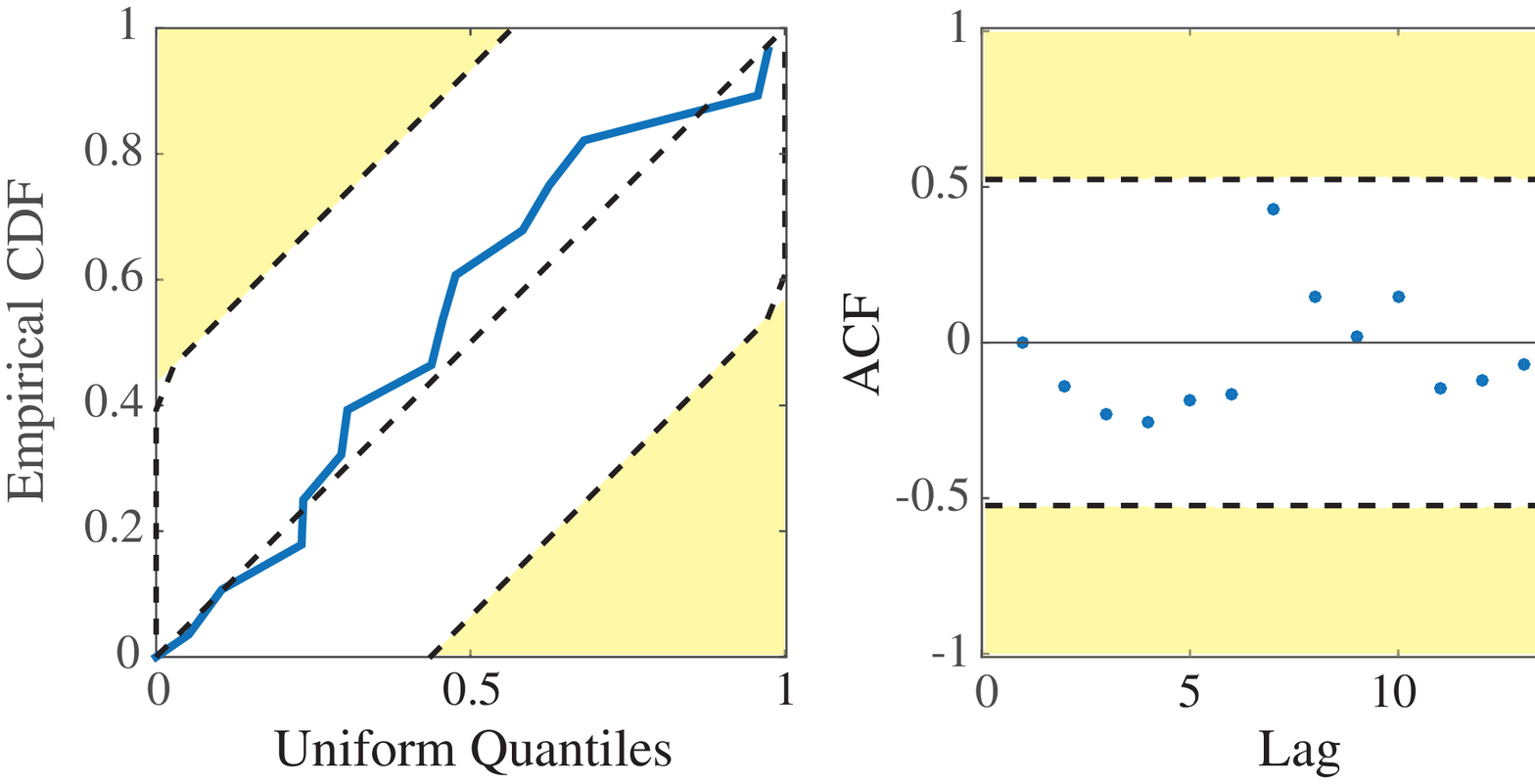}
\subcaption*{(b) $\ell_1$-regularized ML}
\vspace{3mm}
\includegraphics[width=.9\columnwidth]{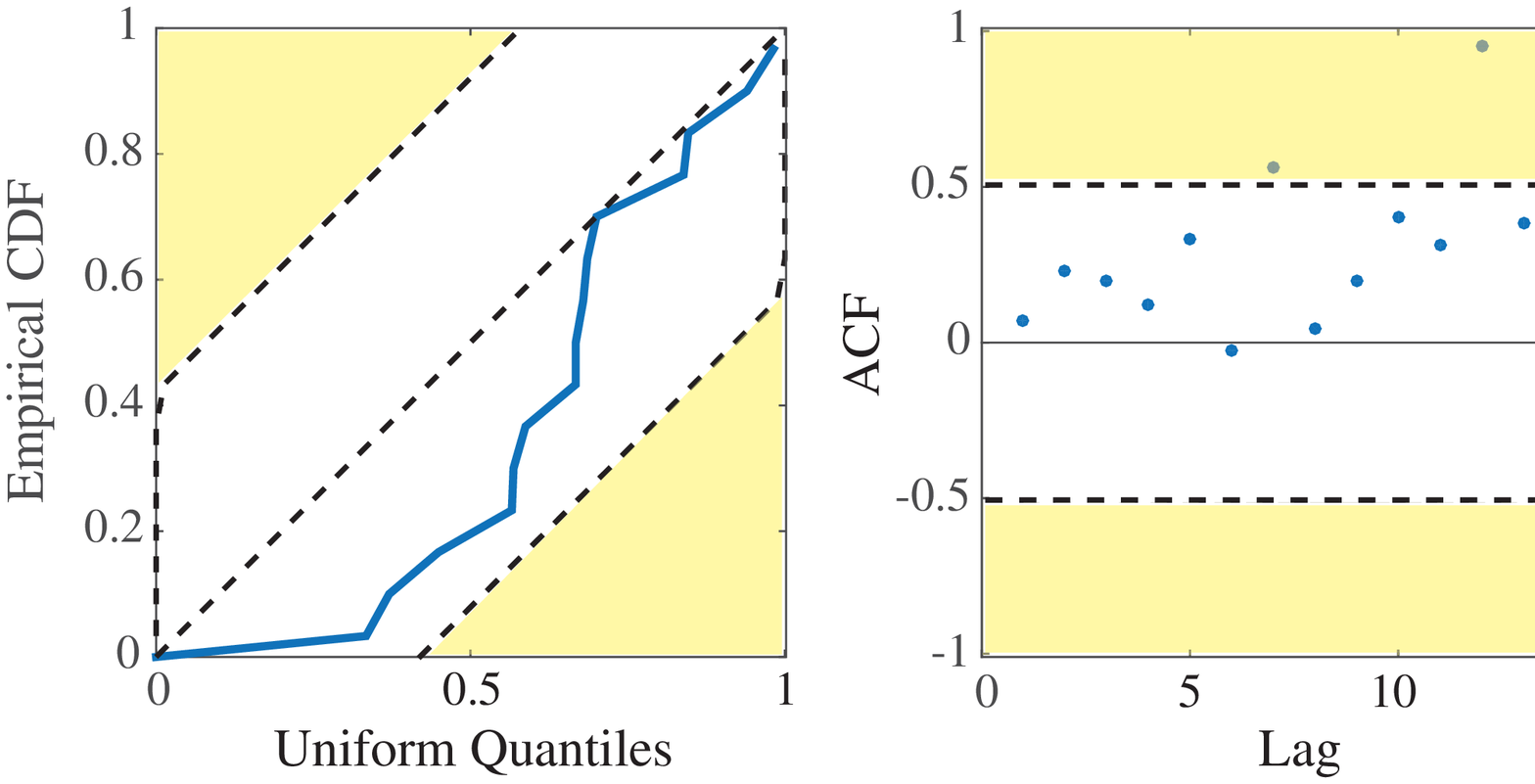}
\subcaption*{(c) POMP}
\end{center}
\vspace{-3mm}
\caption{\small{KS and ACF tests at $99\%$ confidence level, for the ML, $\ell_1$-regularized ML and POMP estimates.}}\label{lgn_ks_sp}
\vspace{-5mm}
\end{figure}

Figure \ref{lgn_mlvssp} shows the estimated history dependence parameter vectors using the three methods. Both the regularized ML (Figure \ref{lgn_mlvssp}(b)) and POMP (Figure \ref{lgn_mlvssp}(c)) estimates capture significant history dependence components around a lag of $90$--$95~\text{ms}$ (marked by the upward arrows). In \cite{Borowska}, an intrinsic neuronal oscillation frequency of around $10 \hertz$ has been reported in around $30\%$ of all classes of mouse retinal cells under experiment, using combined two-photon imaging and patch-clamp recording. Our results are indeed consistent with the above mentioned findings about the intrinsic spiking frequency of retinal neurons. \textcolor{black}{To see this, we consider the power spectral density of the \textcolor{black}{canonical self-exciting process} given by:
\begin{equation}
\label{bart_spec}
S(\omega) = \frac{1}{2\pi} \left( \pi_\star^2 \delta(\omega) + \frac{\pi_\star-\pi_\star^2}{\left(1-\mathbf{1}'\boldsymbol{\theta}\right)^2 \left|1 - \Theta(\omega)\right|^2} \right),
\end{equation}
where $\Theta(\omega)$ is the discrete-time Fourier transform of $\boldsymbol{\theta}$ and $\pi_\star = \mu/(1-\mathbf{1}'\boldsymbol{\theta})$ denotes the stationary distribution probability of spiking. The derivation of the power spectral density is given in Appendix \ref{app:hawkes_psd}. The power spectral density of the \textcolor{black}{canonical self-exciting process} resembles the Bartlett spectrum of the Hawkes process \cite{Bartlett1, Bartlett2, Hawkes_orig}, whose peaks correspond to the significant oscillatory components of the underlying process.} Our estimated parameter vectors $\boldsymbol{\theta}$ using the regularized ML and POMP have significant nonzero components around lags of $ 90 \le k \le 95$. As a result, $S(\omega)$ peaks at $\omega = \frac{2\pi}{k\Delta}$. Hence, $f=\frac{1}{k\Delta}$ is an estimate of the significant intrinsic frequency of the underlying self-exciting process. Using the estimated numerical values, the intrinsic frequency is around $10.5$--$11~\text{Hz}$, which is consistent with experimental findings of \cite{Borowska}. Compared to the method in \cite{Borowska}, our estimates are obtained using much shorter recordings of spiking activity and provide a principled framework to study the oscillatory behavior of LGN neurons using {sparse GLM estimation.}

Note that there is a difference in the orders of magnitudes of the POMP estimate compared to the ML and regularized ML estimates. This is due to the fact that the POMP estimate is exactly $s$-sparse, whereas the ML and regularized ML estimates consist of $p = 100$ non-zero values. In order to assess the goodness-of-fit of these estimates, we invoke the KS and ACF tests. Figure \ref{lgn_ks_sp} shows the corresponding KS and ACF test plots. As it is implied from Figure \ref{lgn_ks_sp}(a), the ML estimate fails both tests due to overfitting, whereas the regularized ML (Figure \ref{lgn_ks_sp}(b)) passes both tests at the specified confidence levels. The POMP estimate (Figure \ref{lgn_ks_sp}(c)), however, passes the KS test while marginally failing the ACF test. The latter observation implies that the seemingly negligible components of the parameter vector captured by the regularized ML estimate seem to be important in explaining the statistics of the observed data.

\subsubsection{{\textbf{Application to RGC spiking activity}}} \label{sec:rgc}
{We will next study the performance of the estimators on spiking} data recorded from the RGCs of neonatal and adult ferrets \cite{wong1993transient}. The retinal ganglion cells are located in the innermost layer of the retina. They integrate information from photoreceptors and project them into the brain \cite{bear2007neuroscience}. The data {were} recorded using a multi-electrode array from the ferret retina at $50~\mu s$ \cite{wong1993transient}. We used $5$ seconds of data from one neuron for the analysis (neuron 2, session 1, adult data set, CARMEN data base \cite{eglen2014data}). {Figure \ref{real_2_fig} shows a segment of the spiking data used in our analysis. The RGC activity in the adult ferret is characterized by bursts of activity with a mean firing rate of $9 \pm 7$~Hz, which are separated by $0.5$--$1~\text{s}$ intervals \cite{wong1993transient}.}

\begin{figure}[t]
\centering
\includegraphics[width=.95\columnwidth]{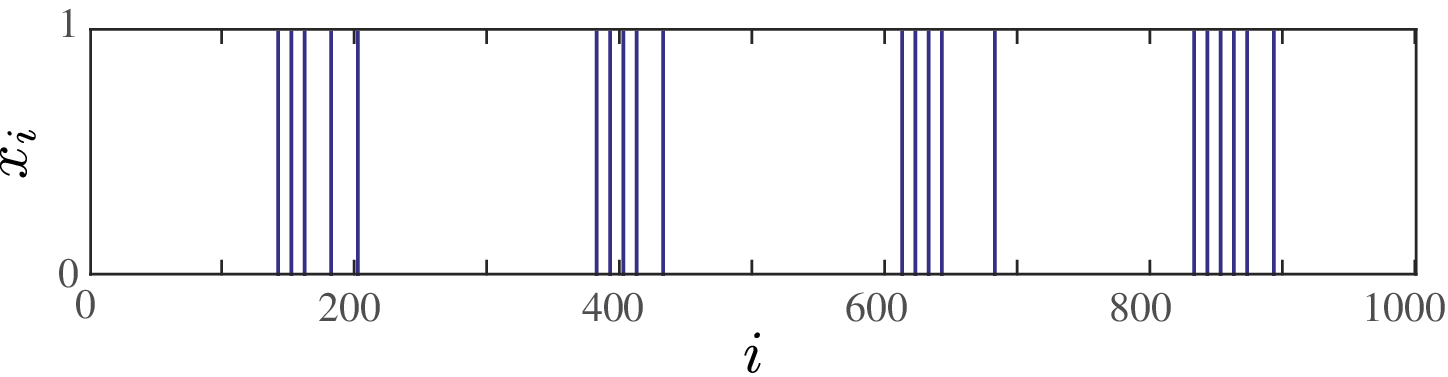}
\caption{\small{Segment of the RGC spiking data used in the analysis.}}\label{real_2_fig}
\vspace{-5mm}
\end{figure} 

\begin{figure}[b!]
\vspace{-4mm}
\begin{center}
\includegraphics[width=.9\columnwidth]{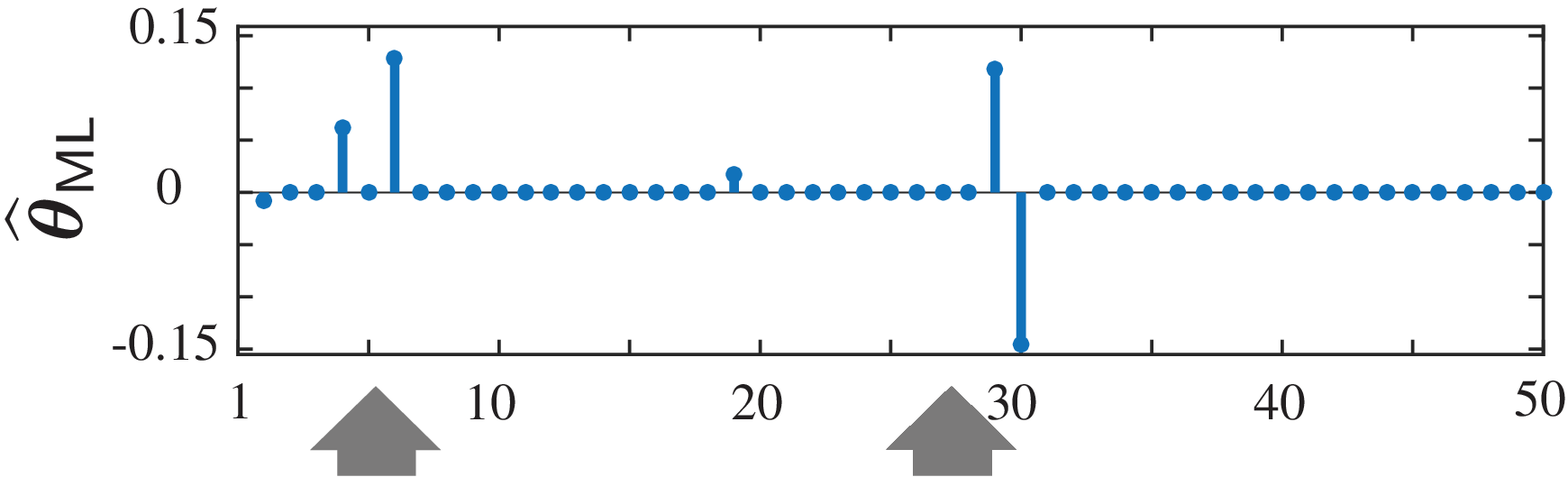}
\subcaption*{(a) ML}
\includegraphics[width=.9\columnwidth]{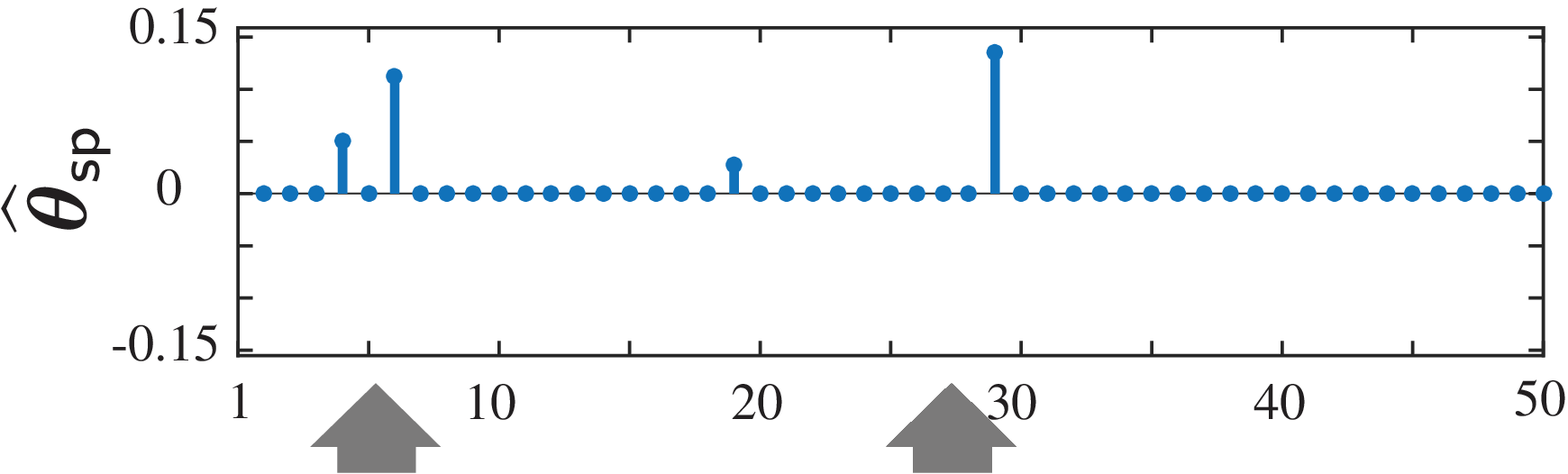}
\subcaption*{(b) $\ell_1$-regularized ML}
\vspace{3mm}
\includegraphics[width=.9\columnwidth]{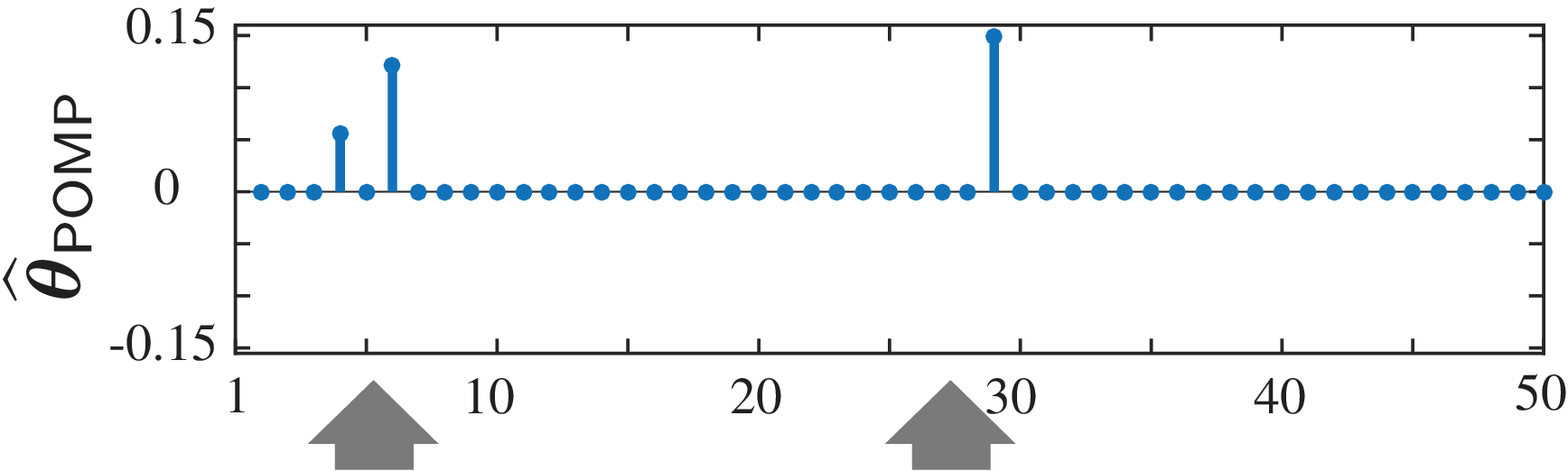}
\subcaption*{(c) POMP}
\end{center}
\vspace{-3mm}
\caption{\small{(a) ML, (b) $\ell_1$-regularized ML, and (c) POMP estimates of the RGC spiking parameters {using the canonical self-exciting process model.}}}
\label{rgc_mlvssp_linear}
\vspace{-2mm}
\end{figure}

In order to capture the history dependence governing the spontaneous spiking activity of the RGC neuron, we model the spiking probability using two different link models to further corroborate the generalization of our results to models beyond the canonical self-exciting process studied in this paper. First, we consider the canonical self-exciting process model. We have chosen  $\pi_\max = 0.49$,  $p=50$ ($\Delta = 25~\text{ms}$) and $s_\star = 3$. The baseline parameter $\mu$ is estimated from the data and is set to be equal to half of empirical mean firing rate of the neuron.

\begin{figure}[t!]
\begin{center}
\includegraphics[width=.9\columnwidth]{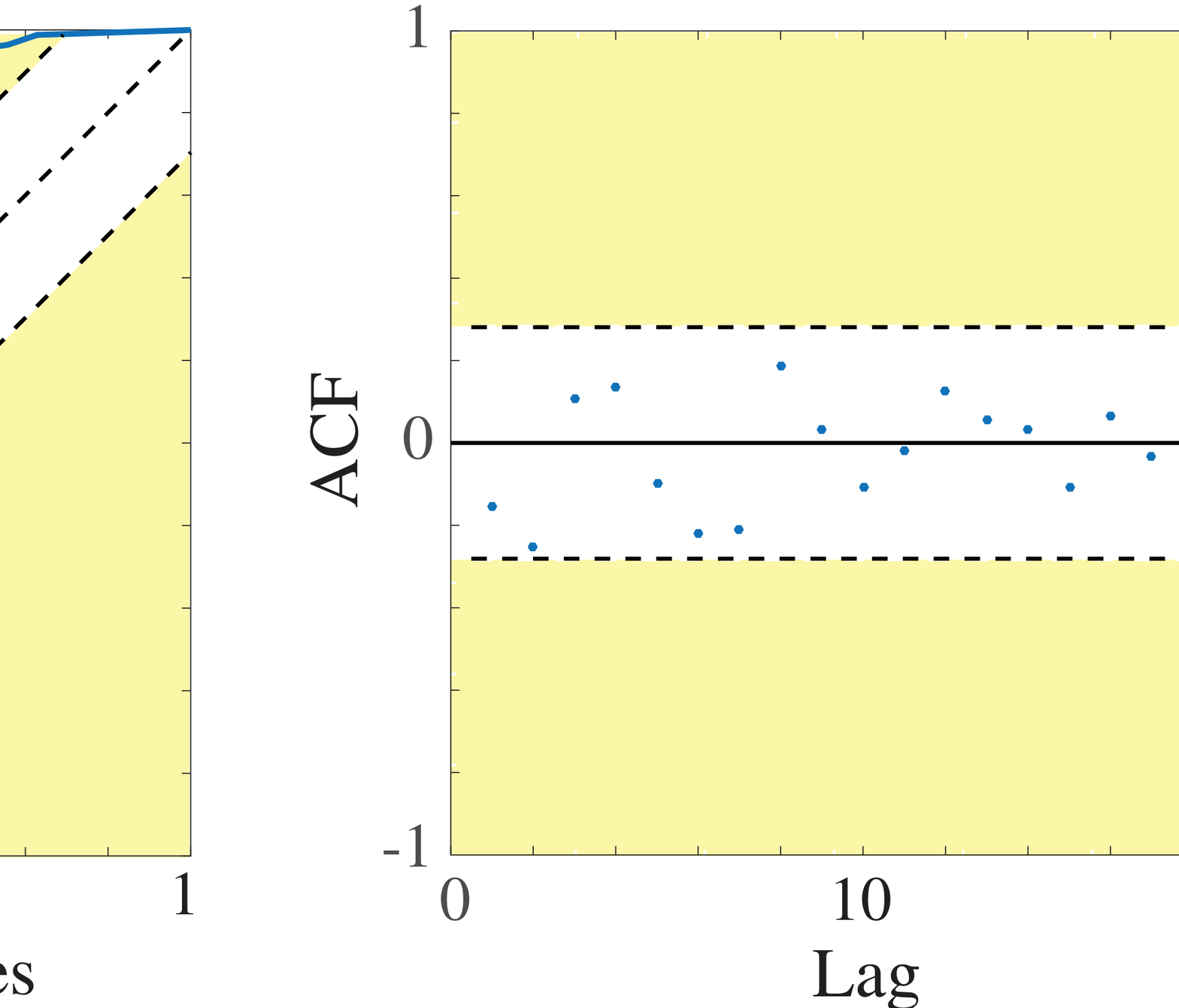}
\subcaption*{(a) ML}
\vspace{3mm}
\includegraphics[width=.9\columnwidth]{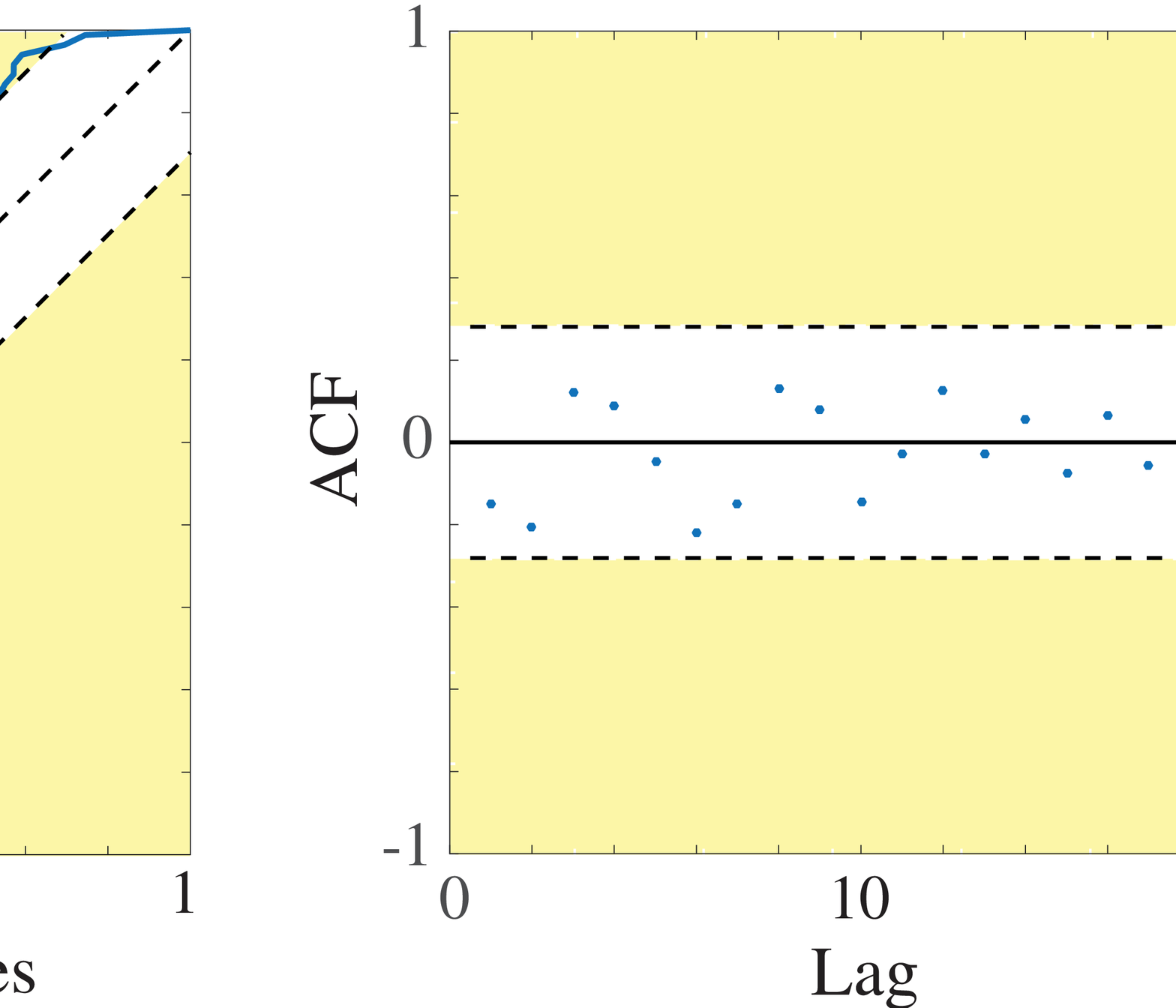}
\subcaption*{(b) $\ell_1$-regularized ML}
\vspace{3mm}
\includegraphics[width=.9\columnwidth]{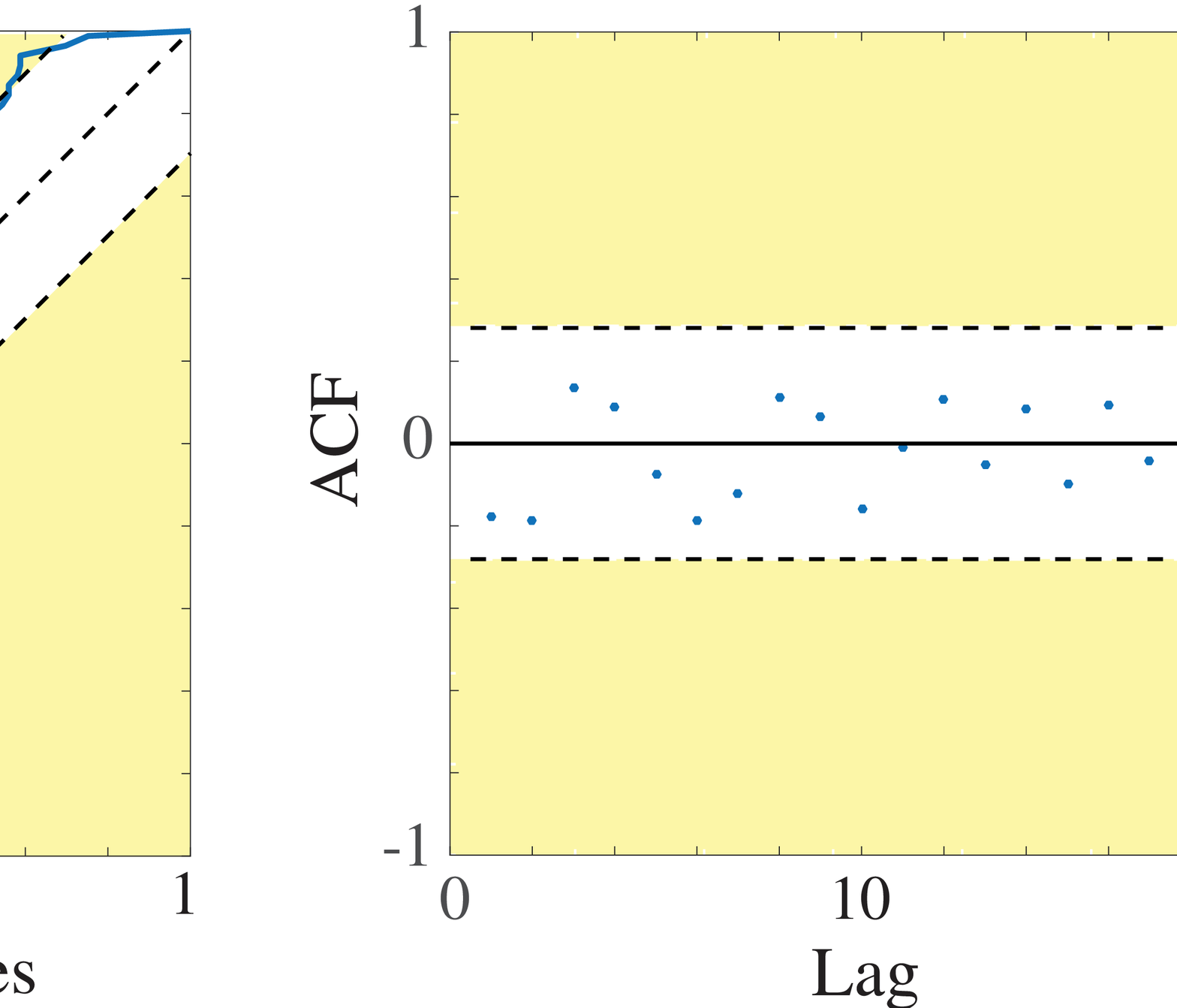}
\subcaption*{(c) POMP}
\end{center}
\vspace{-3mm}
\caption{\small{KS and ACF tests at $95\%$ confidence level, for the ML, $\ell_1$-regularized ML and POMP estimates {using the canonical self-exciting process model.} }}\label{rgc_ks_sp_linear}
\vspace{-6mm}
\end{figure}

\begin{figure}[t!]
\begin{center}
\includegraphics[width=.9\columnwidth]{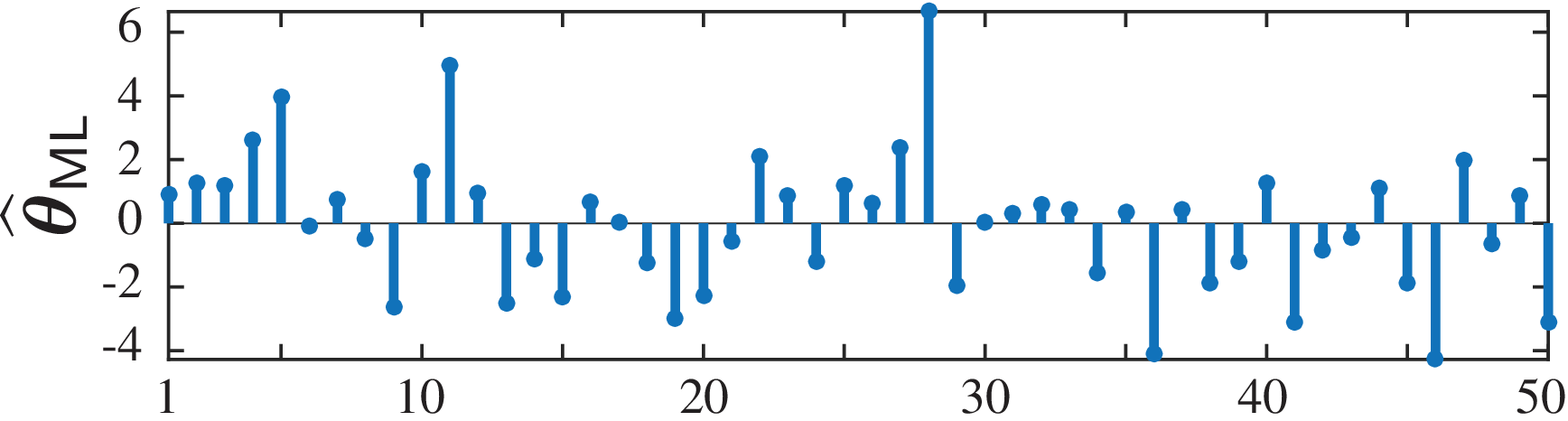}
\subcaption*{(a) ML}
\includegraphics[width=.9\columnwidth]{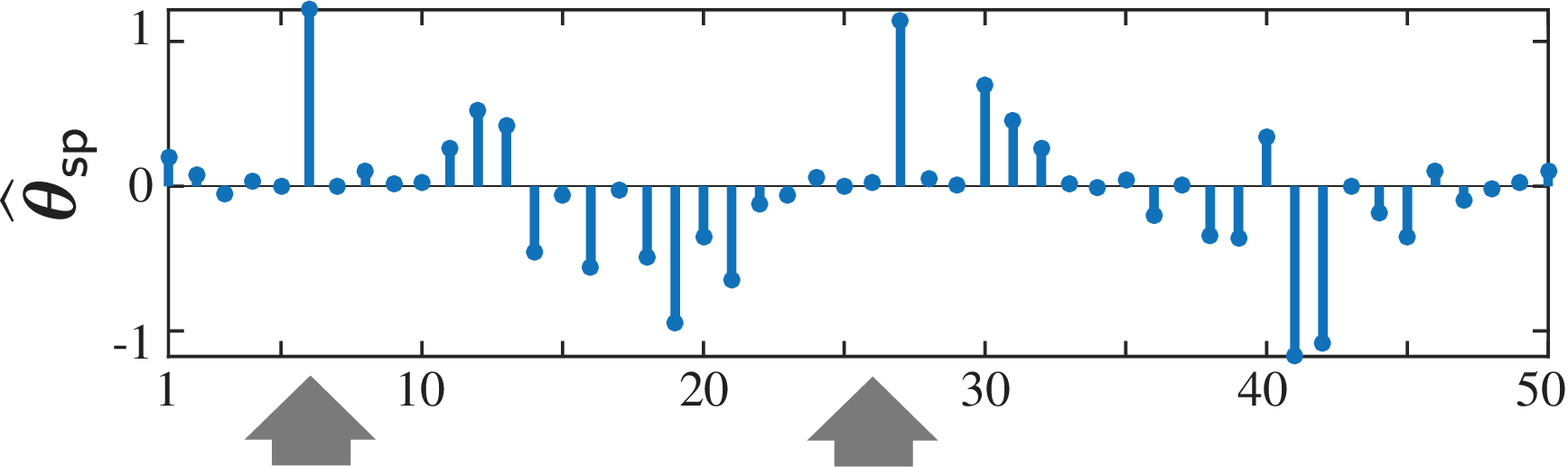}
\subcaption*{(b) $\ell_1$-regularized ML}
\vspace{3mm}
\includegraphics[width=.9\columnwidth]{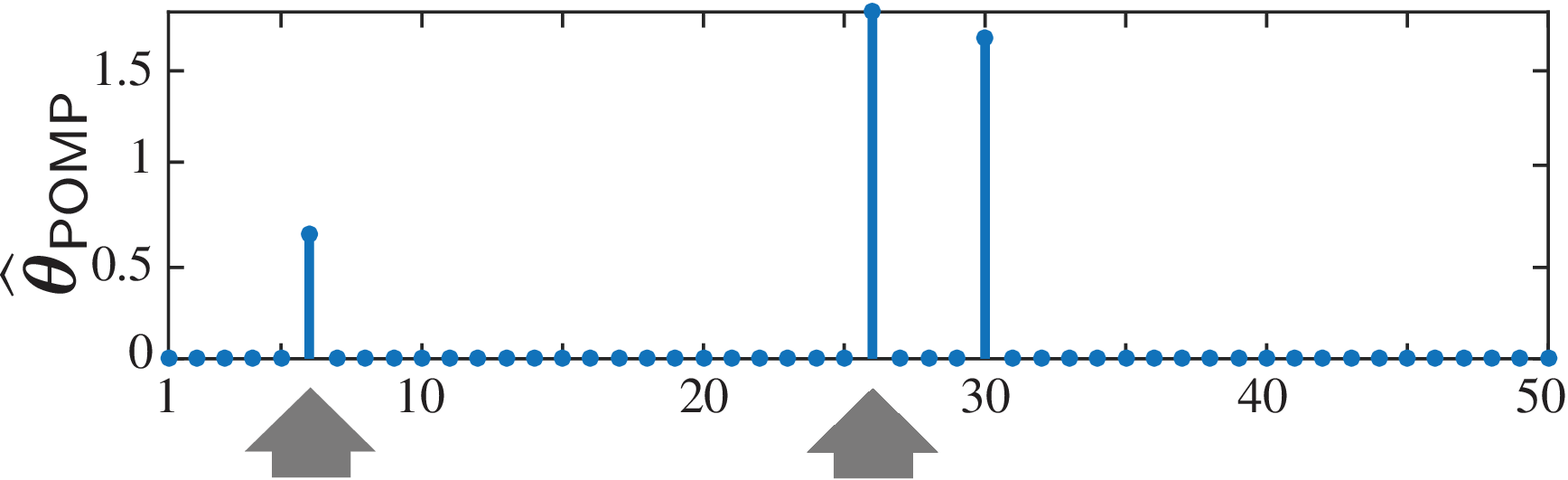}
\subcaption*{(c) POMP}
\end{center}
\vspace{-3mm}
\caption{\small{(a) ML, (b) $\ell_1$-regularized ML, and (c) POMP estimates of the RGC spiking parameters  {using the logistic link model}.}}
\label{rgc_mlvssp}
\vspace{-3mm}
\end{figure}

\begin{figure}[h!]
\vspace{0mm}
\begin{center}
\includegraphics[width=.9\columnwidth]{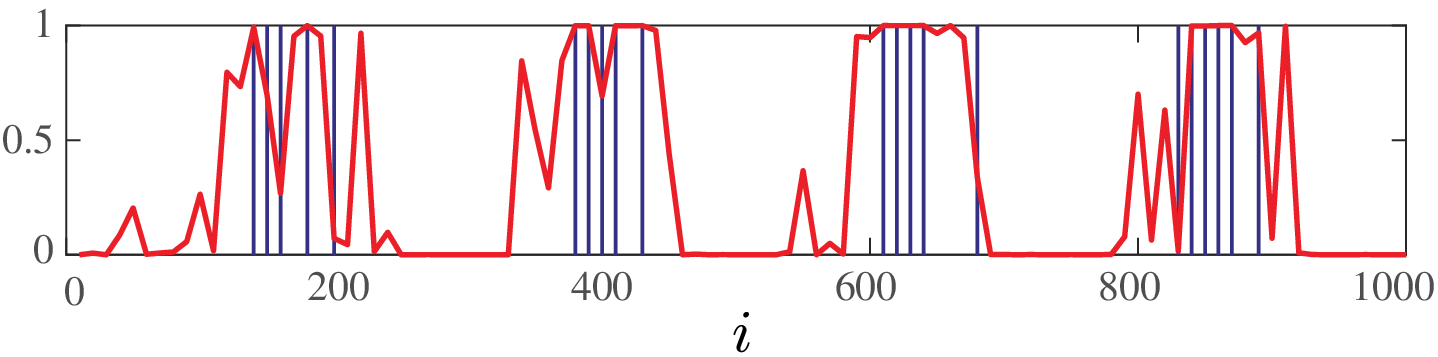}
\subcaption*{(a) ML}
\vspace{2mm}
\includegraphics[width=.9\columnwidth]{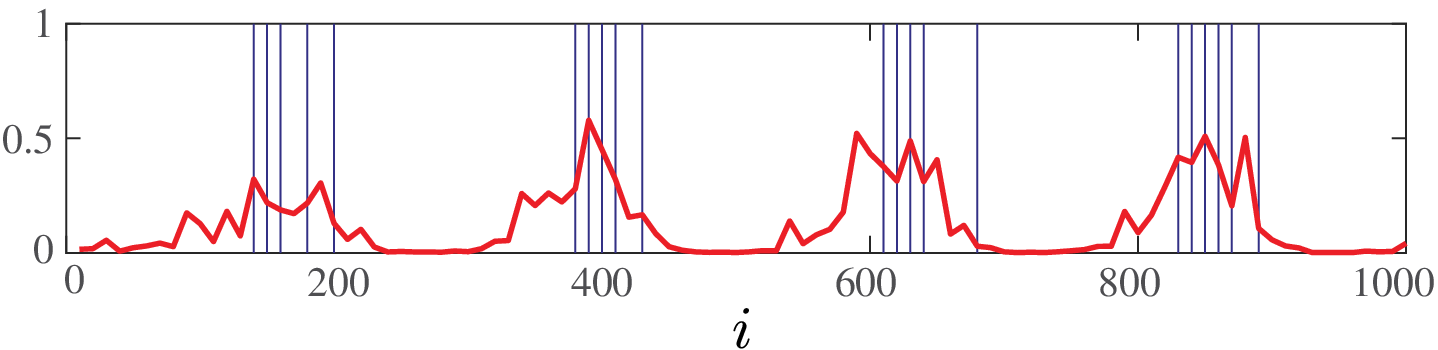}
\subcaption*{(b) $\ell_1$-regularized ML}
\vspace{3mm}
\includegraphics[width=.9\columnwidth]{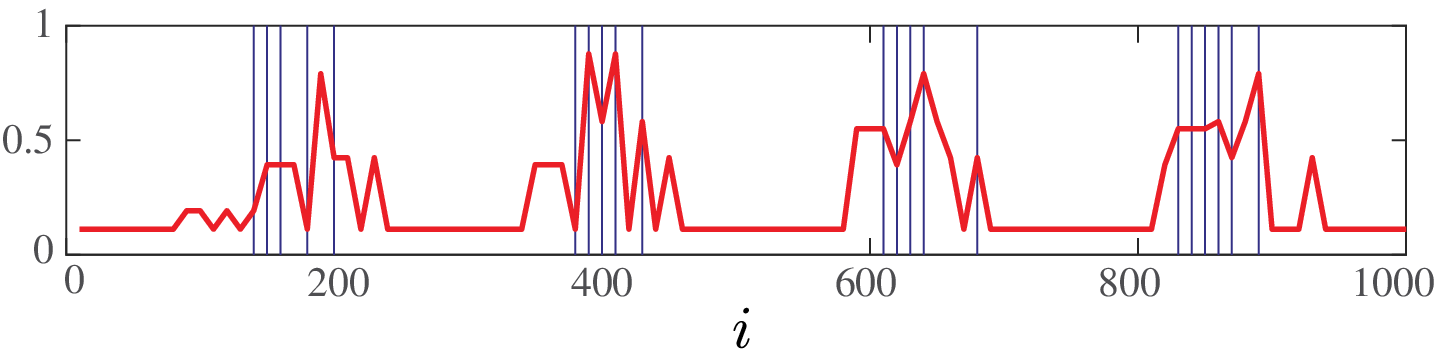}
\subcaption*{(c) POMP}
\end{center}
\vspace{-3mm}
\caption{\small{(a) ML, (b) $\ell_1$-regularized ML, and (c) POMP estimates of the RGC spiking {probability using the unconstrained logistic link model}. Blue vertical lines show the locations of the spikes, and red traces show the estimated {probabilities}.}}
\label{rgc_rates}
\vspace{-5mm}
\end{figure}

Figure \ref{rgc_mlvssp_linear} shows the estimated history components using the three estimators. {All three estimates} capture significant self-exciting history dependence components around the lags of $150~\text{ms}$ and $0.65$--$0.75~\text{s}$ (marked by the upward arrows). Invoking the foregoing argument for the LGN neuron regarding the power spectral density of the process (\ref{bart_spec}), these estimated lag components are consistent with the empirical estimates of \cite{wong1993transient}, as they indicate that the data can be characterized by a combination of $\frac{1}{150~\text{ms}} = 6.66$~Hz bursts separated by gaps of length $0.65$--$0.75~\text{s}$.
{The ML estimator predicts an extra self-inhibitory (negative) component, which results in over-fitting the data.}  This phenomenon can be observed by noting that the ML estimate fails the KS test shown in Figure \ref{rgc_ks_sp_linear}. 

We will next consider a logistic link model of the form $\lambda_i = \frac{\exp(\mu + \boldsymbol{\theta}'x_{i-p}^{i-1})}{C+\exp (\mu + \boldsymbol{\theta}' {\mathbf{x}_{i-p}^{i-1}})}$, with $C = 100$. This model is widely used in neuronal modeling literature (e.g., \cite{truccolo2005point, Brown_pp}), where the assumptions given by (\ref{eq:star}) are dropped and the optimization is performed in an unconstrained fashion. We adopt this approach and obtain all the estimates by dropping the assumptions of $(\star)$. Figure \ref{rgc_mlvssp} shows the estimated history components {using the unconstrained estimators. } {Compared to the canonical self-exciting process model with a linear link, both the regularized ML (Figure \ref{rgc_mlvssp}(b)) and POMP (Figure \ref{rgc_mlvssp}(c)) estimates capture similar significant self-exciting history dependence components, which are consistent across the two sets of estimates.}

{The KS and ACF test results for this case are very similar to Figure \ref{rgc_ks_sp_linear} are are thus omitted for brevity.} In order to further inspect the goodness-of-fit of these methods, we plot the estimated spiking probabilities in Figure \ref{rgc_rates}. The ML estimate shown in Figure \ref{rgc_rates}(a) overfits the spiking events by rapidly saturating the rate to either 0 and 1, which results in undesired high rate estimates where there are no spikes. On the contrary, the regularized ML (Figure \ref{rgc_rates}(b)) and POMP (Figure \ref{rgc_rates}(c)) provide a more reliable estimate of the rates consistent with the spiking events. {This analysis suggests that the sufficient assumptions of $(\star)$ are not necessary for the superior performance of the regularized and POMP estimators over that of ML.}

\section{Conclusion and Future Work}\label{discussions}
In this paper, we studied the sampling properties of $\ell_1$-regularized ML and greedy estimators for a \textcolor{black}{canonical self-exciting process}. The main theorems provide non-asymptotic sampling bounds on the number of measurements, which lead to stable recovery of the parameters of the process. To the best of our knowledge, our results are the first of this kind, and can be readily generalized to various other classes of self-exciting {GLMs}, such as processes with logarithmic or logistic links. 

Compared to the existing literature, our results bring about two major contributions. First, we provide a theoretical underpinning for the advantage of $\ell_1$-regularization in ML estimation as well as greedy estimation in problems involving {binary} observations. These methods have been used in neuroscience in an ad-hoc fashion. Our results establish the utility of these techniques by characterizing the underlying sampling trade-offs. Second, our analysis relaxes the widely-assumed hypotheses of i.i.d. covariates. This assumption is often violated when working with history-dependent data such as neural spiking data. 

We also verified the validity of our theoretical results through simulation studies as well as application to real neuronal spiking data {from mouse's LGN and ferret's RGC neurons.} These results show that both the regularized ML and the greedy estimates significantly outperform the widely-used ML estimate. In particular, through making a connection with the spectrum of discrete point processes, we were able to quantify the estimation of the intrinsic firing frequency of LGN neurons. {In the spirit of easing reproducibility, we have archived a MATLAB implementation of the estimators studied in this work using the CVX package \cite{cvx} on the open source repository GitHub and made it publicly available \cite{hawkes_code}.}

One of the limitations of our analysis is the assumption that the spiking probabilities are bounded by $1/2$, which results in loss of generality. This assumption is made for the sake of theoretical analysis in bounding the mixing rate of the canonical self-exciting process. Our numerical experiments suggest that it is not necessary for the operation of the $\ell_1$-regularized and POMP estimators. We consider further inspection of the mixing properties of this process and thus relaxing this assumption as future work. Our future work also includes generalization of our analysis to multivariate GLMs, which will allow to infer network properties from multi-unit recordings of neuronal ensembles.

\section{Acknowledgement}
We would like to thank L. A. Kontorovich for helpful discussions regarding reference \cite{kontorovich2008concentration}.

\appendices
\section{Proofs of Main Theorems} \label{appprf}

\subsection{Roadmap of the Proofs}
This appendix contains the proofs of Theorems \ref{negahban_lambda} and \ref{thm_OMP}, as well as Corollaries \ref{cor:1} and \ref{cor:2}. Before presenting the proofs, we establish some of the basic properties of the canonical self-exciting process (Proposition \ref{prop:hawkes_properties}) as well as our notational conventions as preliminaries. We then state a key result, namely Lemma \ref{lemma1}, which is at the core of the proofs of Theorems \ref{negahban_lambda} and \ref{thm_OMP}. The proofs are presented via a sequence of three propositions (Propositions \ref{negahban_thm}--\ref{prop_omp}) based on existing results in the literature, in conjunction with Proposition \ref{prop:hawkes_properties} and Lemma \ref{lemma1}. Therefore, Appendix \ref{appprf} is stand-alone modulo the proofs of Proposition \ref{prop:hawkes_properties} and Lemma \ref{lemma1}.

The proofs of Proposition \ref{prop:hawkes_properties} and Lemma \ref{lemma1} are presented in Appendix \ref{app:hawkes_psd}. In particular, the proof of Lemma \ref{lemma1} follows from two propositions (Propositions \ref{eig_conv} and \ref{prop:hawkes_concentration}). Therefore, Appendix \ref{app:hawkes_psd} is stand-alone modulo the proofs of Propositions \ref{eig_conv} and \ref{prop:hawkes_concentration}.

While Proposition \ref{eig_conv} is a well-known result, Proposition \ref{prop:hawkes_concentration} requires a careful proof, which is presented in Appendix \ref{app:c} and relies on an existing result on the concentration of dependent random variables (Proposition \ref{conc_dep}).

\subsection{Preliminaries and Notation}
We state some useful properties of the canonical self-exciting process in the form of the following proposition:
\begin{prop} \label{prop:hawkes_properties}[Properties of the Canonical Self-Exciting Process] The canonical self-exciting process is stationary and we have
\begin{align}
\notag & \pi_\star = \frac{\mu}{1-\mathbf{1}'\boldsymbol{\theta}} >0, \quad  \mu > 0 \Rightarrow  \mathbf{1}' \boldsymbol{\theta}  <1 ,\quad \mu + \mathbf{1}' \boldsymbol{\theta}^+ < 1,\\
\notag & S(\omega) = \frac{1}{2\pi} \left( \pi_\star^2 \delta(\omega) + \frac{\pi_\star-\pi_\star^2}{\left(1-\mathbf{1}'\boldsymbol{\theta}\right)^2 \left|1 - \Theta(\omega)\right|^2} \right),\\
\notag & S(\omega) \ge  \frac{\pi_\star (1 - \pi_\star)}{2 \pi (1+2\pi_\max)^4} =: \kappa_l ,
\end{align}
where $\pi_\star$ denotes the stationary probability of spiking, $S(\omega)$ denotes the power spectral density of the process, and  $\boldsymbol{\theta}^\pm = \max\{\pm \boldsymbol{\theta},\mathbf{0}\}$.
\end{prop}
\begin{proof}
The proof is given in Appendix \ref{app:hawkes_psd}.
\end{proof}

The {stationarity gap of $1-\mathbf{1}'\boldsymbol{\theta}$  plays an important role in controlling the convergence rate of the process to its stationary distribution.}
Throughout the proof, we will also use the notation $\mathbb{S}_p(t) :=\left\{\boldsymbol{\nu}\mid\|\boldsymbol{\nu}\|_p=t\right\}$ to denote the $p$-norm ball of radius $t$.  For simplicity of notation, we also define the $n$-sample empirical expectation as follows:
\[
\widehat{\mathbb{E}}_n \{ f(x_{\cdot}) \} := \frac{1}{n} \sum_{i=1}^n f(x_i)
\]
for any measurable function $f(x_{\cdot})$. Note that the subscript $x_{\cdot}$ refers to an index in the set $\{ 1,2, \cdots, n\}$.

\subsection{Establishing the Restricted Strong Convexity}

The proof of Theorems \ref{negahban_lambda} and \ref{thm_OMP} relies on establishing the Restricted Strong Convexity (RSC) for the negative log-likelihood given by (\ref{L_def}).  Recall that if the log-likelihood is twice differentiable with respect to $\boldsymbol{\theta}$, the {RSC} property implies the existence of a lower quadratic bound on the negative log-likelihood:
\begin{equation}
\label{RSC2}
\mathfrak{D_L}({\boldsymbol{\psi}},\boldsymbol{\theta}):= \mathfrak{L}(\boldsymbol{\theta}+{\boldsymbol{\psi}})-  \mathfrak{L}(\boldsymbol{\theta})-{\boldsymbol{\psi}}'\nabla\mathfrak{L}(\boldsymbol{\theta})\geq \kappa\|{\boldsymbol{\psi}}\|_2^2,
\end{equation}
for a positive constant $\kappa > 0$ and all ${\boldsymbol{\psi}}\in\mathbb{R}^p$ satisfying:
\begin{equation}
\label{eq:cone}
\|\boldsymbol{\psi}_{S^{c}} \|_1 \leq 3 \|{\boldsymbol{\psi}}_S\|_1 + 4 \| \boldsymbol{\theta}_{S^{c}}\|_1.
\end{equation}
for any index set $S \subset \{1,2,\cdots,p\}$ of cardinality $s$. The latter condition is known as the cone constraint. 

The following key lemma establishes the Restricted Strong Convexity condition for the {canonical self-exciting process}:

\begin{lemma}[Restricted strong convexity of the {canonical self-exciting process}]\label{lemma1}
Let ${\mathbf{x}_{-p+1}^n}$ denote a sequence of samples from the {canonical self-exciting process} with parameters $\{\mu,\boldsymbol{\theta}\}$ satisfying the conditions given by (\ref{eq:star}). Then, for $n \ge d_1 s^{2/3} p^{2/3} \log p$, the negative log-likelihood function $\mathfrak{L}(\boldsymbol{\theta})$ satisfies the RSC property with a positive constant $\kappa > 0$ with probability at least $1-2 \exp\left(-\frac{c\kappa^2n^3}{s^2p^2}\right)$, for some constant $c$, and both $\kappa$ and $c$ are only functions of $d_1$, $c_1$, and $\pi_\max$.
\end{lemma}
\begin{proof}
The proof is given in Appendix \ref{app:hawkes_psd}.
\end{proof}

Lemma \ref{lemma1} can be viewed as the key result in the proofs of Theorems \ref{negahban_lambda} and \ref{thm_OMP} which follow next.

\subsection{{Proof of Theorem \ref{negahban_lambda}}}

We first restate the main result of \cite{Negahban} concerning RSC and its implications in controlling the estimation error for GLMs: 
\begin{prop}
\label{negahban_thm}
For a negative log-likelihood $\mathfrak{L}(\boldsymbol{\theta})$ which satisfies the RSC with parameter $\kappa$, every solution to the convex optimization problem (\ref{sp_est_pp_L}) satisfies 
\begin{equation}
\label{error_bound}
\left\|\widehat{\boldsymbol{\theta}}_{\sf {sp}}-\boldsymbol{\theta}\right\|_2 \leq \frac{2\gamma_n \sqrt{s}}{\kappa}+\sqrt{\frac{2\gamma_n \sigma_s(\boldsymbol{\theta})}{\kappa}}
\end{equation}
with a choice of the regularization parameter
\begin{equation}
\label{reg}
\gamma_n \geq   2\left\|\nabla\mathfrak{L}(\boldsymbol{\theta})\right\|_\infty.
\end{equation}
\end{prop}
\begin{proof}
The proof is a special case of {Theorem 1 of \cite{Negahban}}.
\end{proof}

The first term in the bound (\ref{reg}) is increasing in $s$ and corresponds to the estimation error of the $s$ largest components of $\boldsymbol{\theta}$ in magnitude, whereas the second term is decreasing in $s$ and represents the cost of replacing $\boldsymbol{\theta}$ with its best $s$-sparse approximation. 

Given the results of Lemma \ref{lemma1} and Proposition \ref{negahban_thm}, it only remains to establish an upper bound on $\gamma_n$. To this end, we establish a suitable upper bound on $\| \nabla \mathfrak{L}(\boldsymbol{\theta})\|_\infty$ which holds with high probability and provides the appropriate scaling of $\gamma_n$. From Eq. (\ref{L_def}), we have
\begin{equation}
\label{llgrad}
\nabla \mathfrak{L}(\boldsymbol{\theta}) = \frac{1}{n}\sum_{i=1}^n \left[ x_i-(\mu+\boldsymbol{\theta}'  {\mathbf{x}_{i-p}^{i-1}})\right] \frac{{\mathbf{x}_{i-p}^{i-1}}}{{\lambda_i (1-\lambda_i)}}.
\end{equation}
We proceed in two steps:

\medskip
\noindent {\bf Step 1.} We first show that
\begin{equation}
\label{expec=0}
\mathbb{E} \left[ \nabla \mathfrak{L}(\boldsymbol{\theta}) \right]=\mathbf{0}.
\end{equation}
To see this, we use the law of iterated expectations on the $i$th term as follows:
\begin{align}
\label{martingale}
\notag &\mathbb{E} \left[ [x_i-(\mu+\boldsymbol{\theta}'  {\mathbf{x}_{i-p}^{i-1}})]\frac{{\mathbf{x}_{i-p}^{i-1}}}{{\lambda_i (1-\lambda_i)}}\right]\\
\notag &=\mathbb{E} \left[ \mathbb{E} \left[ x_i-(\mu+\boldsymbol{\theta}'  {\mathbf{x}_{i-p}^{i-1}})\frac{{\mathbf{x}_{i-p}^{i-1}}}{{\lambda_i (1-\lambda_i)}}\bigg | {\mathbf{x}_{i-p}^{i-1}} \right]\right]\\
&= \mathbb{E} \left[ \underbrace{\mathbb{E} \left[ x_i-(\mu+\boldsymbol{\theta}'  {\mathbf{x}_{i-p}^{i-1}})\Big | {\mathbf{x}_{i-p}^{i-1}} \right]}_{0} \frac{{\mathbf{x}_{i-p}^{i-1}}}{{\lambda_i (1-\lambda_i)}} \right] = 0
\end{align}
Summing over $i$, establishes (\ref{expec=0}).

\medskip
\noindent {\bf Step 2.} We next show that the summation given by (\ref{llgrad}) is concentrated around its mean. The iterated expectation argument used in establishing (\ref{martingale}) implies that {the sequence
\begin{equation*}
\left\{ \left[ x_i-(\mu+\boldsymbol{\theta}'  {\mathbf{x}_{i-p}^{i-1}})\right] \frac{{\mathbf{x}_{i-p}^{i-1}}}{{\lambda_i (1-\lambda_i)}} \right\}_{i=1}^n
\end{equation*}
}
is a martingale with respect to the filtration given by $\mathcal{F}_{i}=\sigma \left( {\mathbf{x}_{-p+1}^i} \right)$, where $\sigma(\cdot)$ denote the sigma-field generated by the random variables in its argument.
We will now state the following concentration result for sums of bounded and dependent random variables \cite{van_de_geer}:
\begin{prop}
\label{hoeff_dep}
Fix $n\geq 1$. Let $Z_i$'s be bounded $\mathcal{F}_i$-measurable random variables, satisfying for each $i=1,2,\cdots,n$,
\begin{equation*}
\mathbb{E}\left[Z_i|\mathcal{F}_{i-1}\right] = 0, \;\; \text{almost surely}.
\end{equation*}
Then there exists a constant $c$ such that for all $t>0$,
\begin{equation*}
\mathbb{P} \left( \frac{1}{n} \sum_{i=1}^n Z_i - \mathbb{E}[Z_i] \geq t \right)\leq \exp\left(-cnt^2\right).
\end{equation*}
\end{prop}
\begin{proof}
This result is a special case of Theorem 2.5 of \cite{van_de_geer} for \textit{bounded} random variables.
\end{proof}

Proposition \ref{hoeff_dep} implies that  
\begin{equation}
\label{bound}
\mathbb{P}\left( \left | \left ( \nabla \mathfrak{L}(\boldsymbol{\theta})\right)_i\right|  \ge t \right) \leq  \exp(-c nt^2).
\end{equation}
By the union bound, we get:
\begin{equation}
\label{ubound}
\mathbb{P}\Big(\left\| \nabla \mathfrak{L}(\boldsymbol{\theta})\right\|_\infty \ge t \Big) \leq  \exp(-c t^2n+\log p).
\end{equation}
Choosing $t = \sqrt{\frac{{1+\alpha_1}}{c}}\sqrt{\frac{\log p}{n}}$ for some $\alpha_1>0$ yields
\begin{align}
\nonumber \resizebox{\columnwidth}{!}{$\displaystyle \mathbb{P}\left(\left\| \nabla \mathfrak{L}(\boldsymbol{\theta})\right\|_\infty \ge \sqrt{\frac{{1+\alpha_1}}{c}}\sqrt{\frac{\log p}{n}} \right) \leq 2 \exp(-\alpha_1\log p) \leq \frac{2}{n^{\alpha_1}}$}.
\end{align}
Hence, a choice of $\gamma_n = d_2 \sqrt{\frac{\log p}{n}}$ with $d_2 := \sqrt{\frac{1 + \alpha_1}{c}}$ satisfies (\ref{reg}) with probability at least $1 - \frac{2}{n^{\alpha_1}}$.
Combined with the result of Lemma \ref{lemma1} for $n > d_1 s^{2/3} p^{2/3} \log p$, we have that the RSC is satisfied with a constant $\kappa$ with a probability at least $1 - \frac{1}{p^{\alpha_2}} \ge 1 - \frac{1}{n^{\alpha_2}}$ for some constant $\alpha_2$. The latter results in conjunction with Proposition \ref{negahban_thm} establishes the claim of Theorem 1. \QEDB

\noindent \textit{\textbf{Remark.}} The choice of $\pi_\min$ does not affect the proof of Theorem \ref{negahban_lambda}, and can be chosen as $0$ in defining the set $\boldsymbol{\Theta}$, thereby relaxing the first inequality in ($\star$). However, as we will show below, the assumption of $\pi_{\min} > 0$ is required for the proof of Theorem \ref{thm_OMP}.

\subsection{{Proof of Theorem \ref{thm_OMP}}}\label{prf_thm2}

The proof is mainly based on the following proposition, adopted from Theorem 2.1 of \cite{zhang_omp}, stating that the greedy procedure is successful in obtaining a reasonable $s^\star$-sparse approximation, if the cost function satisfies the RSC:
\begin{prop}\label{prop_omp}
Suppose that $\mathfrak{L}(\boldsymbol{\theta})$ satisfies RSC with a constant $\kappa > 0$. Let $s^\star$ be a constant such that
\begin{equation}
\label{sstar}
s^\star \geq \frac{4s}{\pi_\min^2\kappa}\log \frac{20s}{\pi_\min^2 \kappa}= \mathcal{O}(s\log s),
\end{equation}
Then, we have
\begin{equation*}
\left \|\widehat{\boldsymbol{\theta}}^{(s^\star)}_{{\sf POMP}}-\boldsymbol{\theta}_s \right \|_2 \leq \frac{\sqrt{6} \epsilon_{s^\star}}{\kappa},
\end{equation*}
where $\epsilon_{s^\star}$ satisfies
\begin{equation}
\label{eps_bound}
\epsilon_{s^\star} \leq \sqrt{s^\star+s} \|\nabla\mathfrak{L}(\boldsymbol{\theta}_s)\|_\infty .
\end{equation}
\end{prop}
\begin{proof}
The proof is a specialization of the proof of Theorem 2.1 in \cite{zhang_omp} to our setting.
\end{proof}

Recall that Lemma \ref{lemma1} establishes the RSC for the negative log-likelihood function. In order to complete the proof of Theorem \ref{thm_OMP}, it only remains to upper bound $\|\nabla\mathfrak{L}(\boldsymbol{\theta}_s)\|_\infty$. {Let $\lambda_{i,s} := \mu+\boldsymbol{\theta}'_s  \mathbf{x}_{i-p}^{i-1}$.} We have
\begin{align*}
\mathbb{E} \left[\nabla\mathfrak{L}(\boldsymbol{\theta}_s)\right] &= \mathbb{E} \left [ \frac{1}{n}\sum_{i=1}^n \left[ x_i-(\mu+\boldsymbol{\theta}'_s  {\mathbf{x}_{i-p}^{i-1}})\right] \frac{{\mathbf{x}_{i-p}^{i-1}}}{{\lambda_{i,s}(1-\lambda_{i,s})}} \right ]\\
&=\frac{1}{n}\sum_{i=1}^n  \resizebox{.65\columnwidth}{!}{$\displaystyle \mathbb{E} \left [ \mathbb{E} \left [ (\boldsymbol{\theta}-\boldsymbol{\theta}_s)'  {\mathbf{x}_{i-p}^{i-1}} \Big | {\mathbf{x}_{i-p}^{i-1}} \right ] \frac{{\mathbf{x}_{i-p}^{i-1}}}{{\lambda_{i,s}(1-\lambda_{i,s})}} \right ]$}\\
& \le  c_2 \sigma_s(\boldsymbol{\theta}) \mathbf{1}.
\end{align*}
where we have used the fact that $0 \le x_i \le 1$ for all $i$, and {$c_2 := \frac{1}{\pi_\min (1-\pi_\max)}$}. 
Invoking the result of Proposition \ref{hoeff_dep} together with the union bound yields:  
\begin{align*}
\mathbb{P}\left(\|\nabla\mathfrak{L}(\boldsymbol{\theta}_s)\|_\infty \geq  c_1 \sqrt{\frac{\log p}{n}} + c_2 \sigma_s(\boldsymbol{\theta}) \right) \leq \frac{2}{n^{\beta_1}}.
\end{align*}
for some constants $c_1$ and $\beta_1$. Hence, we get the following concentration result for $\epsilon_{s^\star}$:
\begin{equation}
\label{eps_prob}
\mathbb{P}\left(\epsilon_{s^\star} \geq \sqrt{s^\star+s} \left(  c_1 \sqrt{\frac{\log p}{n}} + c_2 \sigma_s(\boldsymbol{\theta})\right)\right) \leq\frac{2}{n^{\beta_1}}.
\end{equation}
Noting that by (\ref{sstar}) we have $s^\star+s = \mathcal{O}(s\log s) \leq c_0 s \log s$, for some constant $c_0$, and invoking the result of Lemma \ref{lemma1}, we get:
\begin{align*}
\left \|\widehat{\boldsymbol{\theta}}^{(s^\star)}_{{\sf POMP}}-\boldsymbol{\theta}_S \right\|_2 &\leq d'_2 \sqrt{\frac{s \log s \log p}{n}} + d'_3 s\log s \sigma_s(\boldsymbol{\theta})\\
&\leq  d'_2 \sqrt{\frac{s \log s \log p}{n}} + d'_3 \frac{\log s}{s^{\frac{1}{\xi}-2}},
\end{align*}
where $d'_2 = \sqrt{c_0} c_1$ and $d'_3 = \sqrt{c_0} c_2$.
with probability $\left(1-\exp\left(-\frac{c\kappa^2n^3}{s^2 (\log s)^2 p^2 }\right)\right)\left(1-\frac{2}{n^{\beta_1}}\right)$. Choosing $n > d'_1 s^{2/3} (\log s)^{2/3} p^{2/3}\log p$ establishes the claimed success probability of Theorem \ref{thm_OMP}. Finally, we have: 
\begin{align*}
\left \|\widehat{\boldsymbol{\theta}}^{(s^\star)}_{{\sf POMP}}-\boldsymbol{\theta}\right\|_2 &= \left \|\widehat{\boldsymbol{\theta}}^{(s^\star)}_{{\sf POMP}}-\boldsymbol{\theta}_s +\boldsymbol{\theta}_s -\boldsymbol{\theta} \right \|_2 \\
& \leq \left \|\widehat{\boldsymbol{\theta}}^{(s^\star)}_{{\sf POMP}}-\boldsymbol{\theta}_s \right \|_2 + \|\boldsymbol{\theta}_s-\boldsymbol{\theta}\|_2.
\end{align*}
Using $\|\boldsymbol{\theta}_s-\boldsymbol{\theta}\|_2 \leq \sigma_s(\boldsymbol{\theta}) = \mathcal{O} \left ( s^{1- \frac{1}{\xi}}\right)$ completes the proof. 

\QEDB

\subsection{{Proofs of Corollaries \ref{cor:1} and \ref{cor:2}}}
\begin{proof}[{Proof of Corollary \ref{cor:1}}]
The claim is a direct consequence of the boundedness of covariates and can be treated by replacing $\boldsymbol{\theta}$ with the augmented parameter vector $[\mu,\boldsymbol{\theta}']'$ and augmenting the covariate vectors with an initial component of 1. The reader can easily verify that all the proof steps can be repeated in the same fashion.
\end{proof}

\begin{proof}[{Proof of Corollary \ref{cor:2}}]
The claim is a direct consequence of the boundedness of covariates which results in $\phi(\cdot)$ being Lipschitz and hence the stationarity of the underlying process. Moreover, for twice-differentiable $\phi(\cdot)$, the proof of Lemma \ref{lemma1} in Appendix \ref{appprf} can be generalized in a straightforward fashion. The reader can easily verify that all the remaining portions of the proofs of the main theorems can be repeated for such $\phi(\cdot)$ in a similar fashion to that of the {canonical self-exciting process}.
\end{proof}

\section{Proofs of Proposition \ref{prop:hawkes_properties} and Lemma \ref{lemma1}}\label{app:hawkes_psd}

\subsection{Proof of Proposition \ref{prop:hawkes_properties}}

The {canonical self-exciting process} can be viewed as a Markov chain with states $X_i= {\mathbf{x}_{i-p}^{i-1}}$. Since each $x_i$ has two possible values, there are $2^p$ possible states. This Markov chain is irreducible since transition from any state to any other state is possible in at most $p$ steps. Also, transition from an all-zero state to itself is possible. Hence the chain is aperiodic as well. This implies that there exists a stationary distribution for the Markov chain. We also know that if $\{X_i\}_{i=1}^\infty$ is a stationary Markov Chain, then for any functional $f(.)$, $\{f(X_i)\}_{i=1}^\infty$ is a strictly stationary stochastic process (SSS). Therefore the {canonical self-exciting process} and the spiking probability sequence $\lambda_1^n$ are both SSS. In particular, we have
\begin{align*}
\pi_\star:= \mathbb{E}[x_i]= \mathbb{E}\left[\mathbb{E}\left[x_i|\lambda_i\right]\right] = \mathbb{E}[\lambda_i]=\mu + \pi_\star\mathbf{1}'\boldsymbol{\theta}.
\end{align*}
Hence, the stationary probability $\pi_\star$ satisfies:
\begin{equation*}
\pi_\star = \frac{\mu}{1-\mathbf{1}'\boldsymbol{\theta}}.
\end{equation*}

In order to prove the first two inequalities, we make the necessary assumption that the baseline rate $\mu$ is positive, due to the non-degeneracy assumption. In order to highlight the necessity of this condition, consider a sample path which contains $p$ successive zeros starting from index $i+1$ to $i+p$, corresponding to an all-zero covariate vector ${\mathbf{x}_{i+1}^{i+p}}$ (note that this sample path will almost surely occur). We then have $\lambda_{i+p+1} = \mu +  \boldsymbol{\theta}' {\mathbf{x}_{i+1}^{i+p}} = \mu$. Therefore, if $\mu$ is not positive, the process becomes degenerate.

The third inequality follows from the fact that for a covariate vector ${\mathbf{x}_{i+1}^{i+p}}$ with a support matching that of $\boldsymbol{\theta}^+$ we have $\lambda_{i+p+1} = \mu +  \boldsymbol{\theta}' {\mathbf{x}_{i+1}^{i+p}} = \mu + \mathbf{1}'\boldsymbol{\theta}^{+}$, which should be a valid probability. Moreover, the inequality is strict since the stationary probability $\pi_\star = \frac{\mu}{1-\mathbf{1}'\boldsymbol{\theta}}$ must be well-defined.

We will next calculate the power spectral density of the process. {Let $\boldsymbol{r}_{-\infty}^{\infty}$ and $\boldsymbol{c}_{-\infty}^\infty$ denote the autocorrelation and autocovariance values of the process, respectively. By the stationarity of the process we have:
\begin{align*}
r_k & = \mathbb{E}\left[x_{\cdot+k} x_\cdot \right]= \mathbb{E}\left[x_k x_{0}\right] = \mathbb{E} \left[ \mathbb{E}\left[x_kx_0 | {\mathbf{x}_{-\infty}^{k-1}} \right]\right]\\
& = \mathbb{E}\left[\mu x_0 + \boldsymbol{\theta}' {\mathbf{x}_{k-p}^{k-1}}x_0  \right] = \mu \pi_\star +  \boldsymbol{\theta}'\boldsymbol{r}_{k-p}^{k-1}.
\end{align*}
for $k > 0$. Similarly, by subtracting the means we have the following identity for the autocovariance:
\begin{equation}
\label{eq:yw_hawkes}
c_k = \boldsymbol{\theta}'\boldsymbol{c}_{k-p}^{k-1}.
\end{equation}
A straightforward calculation gives $c_0 = \pi_\star - \pi_\star^2$. Eq. (\ref{eq:yw_hawkes}) resembles the Yule-Walker equations for an AR process of order $p$ with parameter $\boldsymbol{\theta}$ and the innovations variance given by $\sigma^2 = \frac{\pi_\star-\pi_\star^2}{\left(1-\mathbf{1}'\boldsymbol{\theta}\right)^2}$. Thus, the power spectral density of the \textcolor{black}{canonical self-exciting process} can be expressed as:}
\begin{equation}
S(\omega) = \frac{1}{2\pi} \left( \pi_\star^2 \delta(\omega) + \frac{\pi_\star-\pi_\star^2}{\left(1-\mathbf{1}'\boldsymbol{\theta}\right)^2 \left|1 - \Theta(\omega)\right|^2} \right).
\end{equation}
We have $1-\mathbf{1}'\boldsymbol{\theta} \le 1 + \| \boldsymbol{\theta} \|_1$. Moreover,
\begin{align*}
\resizebox{0.5\columnwidth}{!}{$\displaystyle |1 - \Theta(\omega)| = \left| 1- \sum_k \theta_k e^{-j \omega k} \right|$} &\resizebox{0.5\columnwidth}{!}{$\leq 1 + \|\boldsymbol{\theta}\|_1 =  1 + \|\boldsymbol{\theta}^+\|_1 +  \|\boldsymbol{\theta}^-\|_1$}\\
& \resizebox{0.5\columnwidth}{!}{$\le 1 + 2 (\pi_\max - \mu)  \le 1 + 2 \pi_\max$},
\end{align*}
which implies the lower bound on $S(\omega)$. \QEDB

\subsection{Proof of Lemma \ref{lemma1}}

The proof is inspired by the elegant treatment of Negahban et al. \cite{Negahban}. The major difficulty in the proof lies in the high inter-dependence of the covariates and observations. 

{Noticing that the negative log-likelihood (\ref{L_def}) is twice differentiable}, a second order Taylor expansion of the negative log-likelihood (\ref{L_def}) around $\boldsymbol{\theta}$ yields:
\begin{align*}
\mathfrak{D_L}({\boldsymbol{\psi}},\boldsymbol{\theta})
&= \mathfrak{L}(\boldsymbol{\theta}+{\boldsymbol{\psi}})-\mathfrak{L}(\boldsymbol{\theta})-{\boldsymbol{\psi}}'\nabla \mathfrak{L}(\boldsymbol{\theta})\\
       &= \frac{1}{n}\sum_{i=1}^n x_i \frac{\left({\boldsymbol{\psi}}' {\mathbf{x}_{i-p}^{i-1}}\right)^2}{\left(\mu+\boldsymbol{\theta}'{\mathbf{x}_{i-p}^{i-1}} + \nu({\boldsymbol{\psi}}'{\mathbf{x}_{i-p}^{i-1}})\right)^2}\\
       &  + \frac{1}{n}\sum_{i=1}^n (1-x_i) \frac{\left({\boldsymbol{\psi}}' {\mathbf{x}_{i-p}^{i-1}}\right)^2}{\left(1-\mu-\boldsymbol{\theta}'{\mathbf{x}_{i-p}^{i-1}} - \nu({\boldsymbol{\psi}}'{\mathbf{x}_{i-p}^{i-1}})\right)^2}\\
       & \geq  \frac{1}{n}\sum_{i=1}^n \left({\boldsymbol{\psi}}' {\mathbf{x}_{i-p}^{i-1}}\right)^2,
\end{align*}
for some $\nu \in [0,1]$. The inequality follows from the fact that both $\boldsymbol{\theta}$ and {$\boldsymbol{\theta}+\nu {\boldsymbol{\psi}}$} satisfy (\ref{eq:star}), and hence:
\begin{align}
\nonumber &\mu+\boldsymbol{\theta}'{\mathbf{x}_{i-p}^{i-1}}+ \nu{\boldsymbol{\psi}}'{\mathbf{x}_{i-p}^{i-1}}\leq \pi_\max <1,\\
\nonumber &1 - \mu - \boldsymbol{\theta}'{\mathbf{x}_{i-p}^{i-1}} -  \nu{\boldsymbol{\psi}}'{\mathbf{x}_{i-p}^{i-1}}\leq 1- \pi_{\min} <1.
\end{align}
The result of the Lemma \ref{lemma1} is equivalent to proving that
\begin{equation}
\label{rsceq}
\widehat{\mathbb{E}}_n\left[\left({\boldsymbol{\psi}}'x_{\cdot-p}^{\cdot-1}\right)^2\right] \geq \kappa \|{\boldsymbol{\psi}}\|_2^2
\end{equation}
holds with probability greater than $1-2 \exp\left(-\frac{c\kappa^2n^3}{s^2p^2}\right)$. Since both sides of (\ref{rsceq}) are quadratic in ${\boldsymbol{\psi}}$, the statement is equivalent to proving $\widehat{\mathbb{E}}_n\left[({\boldsymbol{\psi}}' {\mathbf{x}_{\cdot-p}^{\cdot-1}})^2\right] \geq \kappa$, for all $ \|{\boldsymbol{\psi}}\|_2 \in \mathbb{S}_2(1)$. We establish this in two steps:

\medskip
\noindent {\bf Step 1.} First, we show that the statement holds for the true expectation:
\begin{equation}
\label{kappa}
\mathbb{E}\Big[ \left({\boldsymbol{\psi}}' {\mathbf{x}_{\cdot-p}^{\cdot-1}}\right)^2\Big] \geq \kappa_l >0
\end{equation}
for some $\kappa_l$ which will be specified below, for all $ \|{\boldsymbol{\psi}}\|_2 \in \mathbb{S}_2(1)$. To \textcolor{black}{establish} the inequality (\ref{kappa}), we use the following result:
\begin{prop}
\label{eig_conv}
Let ${\mathbf{R}} \in \mathbb{R}^{p \times p}$ be the $p \times p$ covariance matrix of a stationary process with power spectral density $S(\omega)$, and denote its maximum and minimum eigenvalues by $\lambda_{\max}(p)$ and $\lambda_{\min}(p)$ respectively then $\lambda_{\max}(p)$ is increasing in $p$, $\lambda_{\min}(p)$ is decreasing in $p$ and we have
\begin{equation}
\lambda_{\sf min}(p) \downarrow \inf_{\omega}S(\omega), \ \ \text{and} \ \ \lambda_{\sf max}(p) \uparrow \sup_{\omega}S(\omega).
\end{equation}
\end{prop}
\begin{proof}
This is a well-known result in stochastic processes. See \cite{grenander2001toeplitz} for a proof and detailed discussions.
\end{proof}

\noindent Using Proposition \ref{eig_conv}, we can lower-bound $\mathbb{E}\left[\left({\boldsymbol{\psi}}' {\mathbf{x}_{\cdot-p}^{\cdot-1}}\right)^2\right]$ by:
\begin{align*}
&\mathbb{E}\left[ \left({\boldsymbol{\psi}}'{\mathbf{x}_{\cdot-p}^{\cdot-1}}\right)^2\right]=  {\boldsymbol{\psi}}' \mathbf{R} {\boldsymbol{\psi}} \geq \lambda_\min(p) \geq \inf_{\omega}S(\omega).
\end{align*} 
Next, using Proposition \ref{prop:hawkes_properties} the bound of Eq. (\ref{kappa}) follows for
\[
\kappa_l :=  \frac{\pi_\star (1 - \pi_\star)}{2 \pi (1+2\pi_\max)^4}.
\]

\medskip
\noindent {\bf Step 2.} We now show that the empirical and the true expectations of $ \left({\boldsymbol{\psi}}'{\mathbf{x}_{\cdot-p}^{\cdot-1}}\right)^2$ are close enough to each other. Let 
\begin{equation*}
\mathfrak{D}_{{\boldsymbol{\psi}},n} := \widehat{\mathbb{E}}_n\left[ \left({\boldsymbol{\psi}}'{\mathbf{x}_{\cdot-p}^{\cdot-1}}\right)^2\right]-\mathbb{E}\left[ \left({\boldsymbol{\psi}}'{\mathbf{x}_{\cdot-p}^{\cdot-1}}\right)^2\right].
\end{equation*}
\noindent and
\begin{equation*}
\mathfrak{D}_{n} := \sup \limits_{{\boldsymbol{\psi}} \in \mathbb{S}_2(1)} \left| \mathfrak{D}_{{\boldsymbol{\psi}},n} \right |.
\end{equation*}
The final step in proving Lemma \ref{lemma1} is given by the following proposition:
\begin{prop}\label{prop:hawkes_concentration} We have
\begin{equation}
\label{zbounded}
\mathbb{P}\left[\mathfrak{D}_{n} \geq \frac{\kappa_l}{4}\right] \leq 2 \exp\left(-\frac{c\kappa_l^2 n^3}{s^2p^2}\right),
\end{equation}
for some constant $c$.
\end{prop}

\begin{proof}
{The proof is given in Appendix \ref{app:c}.}
\end{proof}

{Finally, the statement of Lemma \ref{lemma1} follows from Proposition \ref{prop:hawkes_concentration} by taking $\kappa = \kappa_l / 4$. \QEDB}

\section{Proof of Proposition \ref{prop:hawkes_concentration}}\label{app:c}

{In order to establish the concentration inequality of Eq. (\ref{zbounded}), we need to invoke a result from concentration of dependent random variables. We proceed in two steps:}

\medskip
\noindent {{\bf Step 1.}} {We first establish a geometric property of $\mathfrak{D}_n$, namely its $\mathcal{O}(\frac{sp}{n})$-Lipschitz property with respect to the normalized Hamming metric. Recall that the normalized Hamming metric between two sequences ${\mathbf{x}_1^n}$ and ${\mathbf{y}_1^n}$ is defined as $d({\mathbf{x}_1^n,\mathbf{y}_1^n)} = \frac{1}{n}\sum_{i=1}^n\mathbf{1}(x_i \neq y_i)$.}

First, by evaluating the first order optimality conditions of the solution $\widehat{\boldsymbol{\theta}}_{{\sf sp}}$, it can be shown that the error vector ${\boldsymbol{\psi}} = \widehat{\boldsymbol{\theta}}_{\sf sp}-\boldsymbol{\theta}$ satisfies the inequality:
\begin{equation*}
{\|\boldsymbol{\psi}_{S^{c}} \|_1 \leq 3 \|{\boldsymbol{\psi}}_S\|_1 + 4 \| \boldsymbol{\theta}_{S^{c}}\|_1,}
\end{equation*}
{with $S$ denoting the support of the best $s$-term approximation to $\boldsymbol{\theta}$} (see for example \cite{Negahban}). By the assumption of $\sigma_S(\boldsymbol{\theta}) = \mathcal{O}(\sqrt{s})$, we can choose a constant $c_0$ such that $\sigma_S(\boldsymbol{\theta}) \le c_0 \sqrt{s}$. Hence,
\begin{align}
\label{cone}
\resizebox{0.89\columnwidth}{!}{$\|{\boldsymbol{\psi}}\|_1 \leq 4 \|{\boldsymbol{\psi}}_S\|_1  + \sigma_s(\boldsymbol{\theta}) \leq (4 + c_0)\sqrt{s}\|{\boldsymbol{\psi}}_S\|_2 \leq (4 + c_0) \sqrt{s}$}
\end{align}
where we have used the fact that $\| {\boldsymbol{\psi}}_S \|_1 \le \sqrt{s} \| {\boldsymbol{\psi}}_S \|_2 \le \sqrt{s}$ for all ${\boldsymbol{\psi}} \in \mathbb{S}_2(1)$. Therefore for all $i \in \{ 1,2,\cdots, n\}$, we have:

\begin{equation}
\label{fbound}
0 \leq \left({\boldsymbol{\psi}}'{\mathbf{x}_{i-p}^{i-1}}\right)^2 \leq \left\|{\boldsymbol{\psi}}\right\|_1^2 \leq (4 + c_0)^2 s.
\end{equation}

We first prove the claim for $\mathfrak{D}_{{\boldsymbol{\psi}},n}$. To establish the latter, we need to prove
\[
\frac{1}{n}\left|\sum_{i=1}^n \left({\boldsymbol{\psi}}'{\mathbf{x}_{i-p}^{i-1}}\right)^2-  \left({\boldsymbol{\psi}}'{\mathbf{y}_{i-p}^{i-1}}\right)^2\right| \leq C d({\mathbf{x}_{-p+1}^{n}},{\mathbf{y}_{-p+1}^{n}}),
\]
for some $C = \mathcal{O}(\frac{sp}{n})$, or equivalently
\begin{equation}
\nonumber \left|\sum_{i=1}^n  \left({\boldsymbol{\psi}}'{\mathbf{x}_{i-p}^{i-1}}\right)^2- \left({\boldsymbol{\psi}}'{\mathbf{y}_{i-p}^{i-1}}\right)^2\right| \leq C' \sum_{i=-p+1}^n\mathbf{1}(x_i \neq y_i),
\end{equation}

for some $C' = \mathcal{O}(s)$.
Let us start by setting the values of ${\mathbf{x}_{-p+1}^{n}}$ equal to those of ${\mathbf{y}_{-p+1}^{n}}$ and iteratively change $x_j$ to $1-x_j$ for all indices $j$ where $x_j \neq y_j$ to obtain the configuration given by ${\mathbf{x}_{-p+1}^{n}}$.  For each such change (say $x_j$ to $1-x_j$), the left hand side changes by at most
\begin{align*}
&\left| \sum_{i=1}^n \left({\boldsymbol{\psi}}'{\mathbf{x}_{i-p}^{i-1}}\right)^2_{|x_j =1} -  \left({\boldsymbol{\psi}}'{\mathbf{x}_{i-p}^{i-1}}\right)^2_{|x_j =0} \right| \\
& \leq   \left({\boldsymbol{\psi}}'{\mathbf{x}_{j-p}^{j-1}}\right)^2 + 2\sum_{i\neq j}  |\psi_{i-j}| \|{\boldsymbol{\psi}}\|_1 \leq 3 \|{\boldsymbol{\psi}}\|_1^2 \leq 3(4+c_0)^2 s,
\end{align*}
where we have used the inequality given by Eq. (\ref{fbound}). Hence, the $C$ can be taken as $3 (4 + c_0)^2 s p / n$ and the claim of the proposition for $\mathfrak{D}_{{\boldsymbol{\psi}},n}$ follows.  A very similar argument can be used to extend the claim to $\mathfrak{D}_n$. Let ${\boldsymbol{\psi}}^\star := {\boldsymbol{\psi}}^\star({\mathbf{x}_{-p+1}^{n}})$ be the ${\boldsymbol{\psi}}$ for which the supremum in the definition of $\mathfrak{D}_n$ is achieved (such a choice of ${\boldsymbol{\psi}}$ exists by the Weierstrass extreme value theorem). Since ${\boldsymbol{\psi}}^{\star}$ also satisfies (\ref{cone}), a similar argument shows that $\mathfrak{D}_n$ is $\mathcal{O}(\frac{sp}{n})$-Lipschitz (with possibly different constants). 

\medskip
\noindent {{\bf Step 2.} Next, we establish the concentration of $\mathfrak{D}_n$ around zero.} Let $H = [{\mathbf{x}_{i-p}^{i-2}},1]$ and $\widehat{H} = [{\mathbf{x}_{i-p}^{i-2}},0]$ be two vectors (history components) of length $p$ which only differ in their last component, and let the mixing coefficient $\bar{\eta}_{ij}$ for $j \ge i$ be defined as:
\begin{equation}
\label{eta}
\bar{\eta}_{ij} = \|p({\mathbf{x}_j^n}|H)-p({\mathbf{x}_j^n}|\widehat{H})\|_{TV},
\end{equation}
with $\| \cdot \|_{TV}$ denoting the total variation difference of the probability measures {induced} on $\{0,1\}^{n-j+1}$. Also, let
\[
\eta_{ij} = \sup_{H,\widehat{H}} \bar{\eta}_{ij}, \quad \mbox{and} \quad Q_{n,i} := 1+ \eta_{i,i+1}+ \cdots + \eta_{i,n}.
\]
We now invoke Theorem 1.1 of \cite{kontorovich2008concentration} in the form of the following proposition:
\begin{prop}
\label{conc_dep}
If $\mathfrak{D}_n$ is $C$-Lipschitz and $q := \max_{1 \leq i \leq n} Q_{n,i}$, then
\begin{equation*}
\mathbb{P}\left [| \mathfrak{D}_n - \mathbb{E}[\mathfrak{D}_n] | \geq t \right] \leq 2 \exp \left(\ \frac{-2nt^2}{qC^2} \right).
\end{equation*}
\end{prop}
\begin{proof}
The proof is identical to the beautiful treatment of \cite{kontorovich2008concentration} when specializing the underlying function of the variables ${\mathbf{x}_{-p+1}^i}$ to be $\mathfrak{D}_n$.
\end{proof}

As we showed in Step 1, $C = C' s p / n$, for some constant $C'$. Now, we have
\begin{align*}
\eta_{ij} \leq 2^{n-j+1} |\pi_{\max}^{n-j+1} -\pi_{\min}^{n-j+1}| \leq \left(2\pi_{\max}\right)^{n-j+1},
\end{align*}
where we have used the fact that each element of the measures $p({\mathbf{x}_j^n} | H)$ and $p({\mathbf{x}_j^n} | \widehat{H})$ satisfies the assumption (\ref{eq:star}) and that the size of the state space $\{0,1\}^{n-j+1}$ is given by $2^{n-j+1}$. By the assumption (\ref{eq:star}), we have $\eta_{ij} \le \rho^{n-j+1}$ for $\rho := 2 \pi_\max < 1$. Hence, $Q_{n,i} \le \frac{1}{1-\rho}$ for all $i$, and $q \le \frac{1}{1-\rho}$ by definition. Using the result of Proposition \ref{conc_dep}, we get:
\begin{equation}
\label{debound}
\mathbb{P}\left[\mathfrak{D}_{n} \ge \mathbb{E}[\mathfrak{D}_n] + \frac{\kappa_l}{2} \right ] \le 2 \exp \left(\ \frac{-n^3 \kappa_l^2 ( 1 - \rho)}{2 C' s^2 p^2} \right). 
\end{equation}
It only remains to show that the expectation in (\ref{debound}) can be suitably bounded. {Note that by a similar concentration argument for} $\mathfrak{D}_{ {\boldsymbol{\psi}}^\star,n}$, we have:
\begin{align*}
\mathbb{E}[\mathfrak{D}_n]&= \mathbb{E}[|\mathfrak{D}_{ {\boldsymbol{\psi}}^\star,n}|] = \int_0^\infty \left(1-F_{\left|\mathfrak{D}_{{\boldsymbol{\psi}}^\star,n}\right|}(t)\right)dt\\
& \leq \int_0^\infty 2 \exp\left(-\frac{2 (1-\rho) n^3t^2}{C' s^2 p^2}\right)dt = 2\sqrt{\frac{C'\pi}{(1-\rho)}}\frac{ps}{n^{3/2}}.
\end{align*}
Thus choosing $n \ge d_1 s^{2/3}p^{2/3} \log p$, for some positive constant $d_1$, $\mathbb{E}[\mathfrak{D}_n]$ drops as $1/\log^{3/2} p$, and will be smaller than $\kappa_l / 4$ for large enough $p$. Hence, combined with (\ref{debound}) and by defining $c := \frac{1-\rho}{2 C'}$ we have:
\begin{equation*}
\mathbb{P}\left[\mathfrak{D}_{n} \ge \frac{\kappa_l}{4} \right ] \le 2 \exp \left(\ \frac{-c n^3 \kappa_l^2}{s^2 p^2} \right),
\end{equation*}
which establishes the claim of Proposition \ref{prop:hawkes_concentration}. \QEDB

\section{Goodness-Of-Fit Tests for Point Process Models}\label{appks}

In this appendix, we will give an overview of the statistical tools used to assess the goodness-of-fit of point process models. A detailed treatment can be found in \cite{Brown_pp}.

\medskip
\noindent \textit{\textbf{The Time-Rescaling Theorem.}} Let $0<t_1<t_2<\cdots$ be a realization of a continuous point process with conditional intensity $\lambda(t)>0$, i.e. $t_k$ is the first instance at which $N(t_k)=k$. Define the transformation
\begin{equation}
\label{trt}
z_k := Z(t_k)= \int_{t_{k-1}}^{t_k} \lambda(t) dt.
\end{equation}
Then, the transformed point process with events occurring at $t'_k = \sum_{i=1}^k z_k$ corresponds to a {homogeneous} Poisson process with rate 1. Equivalently,  $z_1,z_2,\cdots$ are \textit{i.i.d} \textit{exponential} random variables. The latter can be used to construct statistical tests for the goodness-of-fit.

\noindent \textit{\textbf{The Komlogorov-Smirnov Test for Homogeneity.}} Suppose that we have obtained the rescaled process through (\ref{trt}) with the \textit{estimated} conditional intensity. When applying the time-rescaling theorem to the discretized process, if the estimated conditional intensity is close to its true value, the rescaled process is expected to behave as a {homogeneous} Poisson process with rate $1$. The Kolmogorov-Smirnov (KS) test can be used to check for the homogeneity of the process. Let $z_k$'s be the rescaled times and define the transformed rescaled times by the inverse exponential CDF $u_k:= 1-e^{-z_k}$. If the true conditional intensity was used to rescale the process, the random variables $u_k$ must be i.i.d. ${\sf Uniform}(0,1]$ distributed. The KS test plots the empirical qualities of $u_k$'s versus the true quantiles of the uniform density given by $b_k = \frac{k-1/2}{J}$, where $J$ is the total number of observed spikes. If the conditional intensity is well estimated, the resulting curve  must lie near the $45^\circ$ line. The asymptotic statistics of the KS distribution can be used to construct confidence intervals for the test. For instance, the $95\%$ and $99 \%$ confidence intervals are approximately given by $\pm \frac{1.36}{\sqrt{J}}$ and $\pm \frac{1.63}{\sqrt{J}}$ hulls around the $45^\circ$ line, respectively.

\medskip
\noindent \textit{\textbf{The Autocorrelation Function Test for Independence.}} In order to check for the independence of the resulting rescaled intervals $z_k$, the transformation $v_k = \Phi^{-1}(u_k)$ is used, where $\Phi$ is the standard Normal CDF. If the true conditional intensity was used to rescale the process, then $v_k$'s would be i.i.d. Gaussian and their uncorrelatedness would imply independence. The Autocorrelation Function (ACF) of the variables $v_k$ must then be close to the discrete delta function. The $95\%$ and $99 \%$ confidence intervals can be considered using the asymptotic statistics of the sample ACF, approximately given by $\pm \frac{1.96}{\sqrt{J}}$ and $\pm \frac{2.575}{\sqrt{J}}$, respectively.

\medskip
\noindent \textit{\textbf{Remark.}} The binning size used for discretizing the data can potentially affect the ISI distribution of the time-rescaled process. In order to avoid these issues, we have used the empirical ISI distribution estimated from a large realization of the process (estimated from the training data) as the null hypothesis for both tests (performed on the test data).

{
\small
\bibliographystyle{IEEEtran}
\bibliography{Robust_SEGLM}
}

\end{document}